%% file: main.tex
\documentclass[10pt,twocolumn,letterpaper]{article}

\usepackage{cvpr}      %

\usepackage{times}
\usepackage{epsfig}
\usepackage{graphicx}
\usepackage{amsmath}
\usepackage{amssymb}
\usepackage{amsthm}
\usepackage{mathtools}
\usepackage{bm}
\makeatletter
\@namedef{ver@everyshi.sty}{}
\makeatother
\usepackage[mode=buildnew]{standalone}
\usepackage{tikz}
\usetikzlibrary{positioning}
\usetikzlibrary{shapes} 
\usetikzlibrary{matrix}
\usetikzlibrary{calc} 
\usetikzlibrary{intersections}
\usepackage[accsupp]{axessibility}  %
\usepackage{tikzscale}
\usepackage{pgfplots}
\usepackage{xcolor}
\usepackage[eulergreek]{sansmath}
\usepackage{caption,subcaption}
\usepackage{bbm}
\usepackage{comment}
\usepackage{acro}
\usepackage{listings}
\usepackage{multirow}
\usepackage{tabularx}
\usepackage{booktabs}
\usepackage{wrapfig}
\usepackage[ruled, vlined, linesnumbered]{algorithm2e}
\usepackage{float}

\usepackage{cuted}
\usepackage{pbox}

\usepackage[pagebackref=true,breaklinks=true,letterpaper=true,colorlinks,bookmarks=false]{hyperref}
\usepackage[capitalise]{cleveref}

\renewcommand{\paragraph}[1]{\textbf{#1}}

\usepackage{enumitem}
\setitemize{noitemsep,topsep=0pt,parsep=0pt,partopsep=0pt,leftmargin=1em}
\setenumerate{noitemsep,topsep=0pt,parsep=0pt,partopsep=0pt,leftmargin=1em}

\usepackage{xcolor}

\def\confName{CVPR}
\def\confYear{2022}

\newcommand{\R}{\mathbb{R}}

\newcommand{\la}{\langle}
\newcommand{\ra}{\rangle}
\DeclarePairedDelimiter\abs{\lvert}{\rvert}%

\newtheorem{theorem}{Theorem}

\newtheorem{prop}[theorem]{Proposition}

\newtheorem{definition}[theorem]{Definition}

\input{setup/acros}
\begin{document}

\title{A Scalable Combinatorial Solver for Elastic \\Geometrically Consistent 3D Shape Matching}

\author{Paul Roetzer$^{1,2}$ $\qquad$ Paul Swoboda$^3$ $\qquad$ Daniel Cremers$^1$ $\qquad$ Florian Bernard$^2$\\
TU Munich$^1$ $\qquad$ University of Bonn$^2$ $\qquad$ MPI Informatics$^3$ 
}

\newcommand{\teaserheight}[0]{3.05cm}
\maketitle%
\begin{strip}%
  \centerline{%
  \footnotesize%
  \begin{tabular}{ccc}%
          \includegraphics[height=\teaserheight]{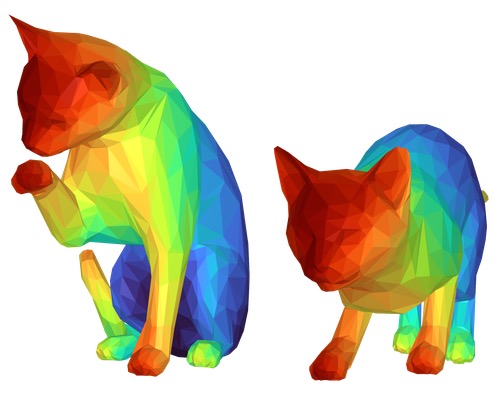} &
          \includegraphics[height=\teaserheight]{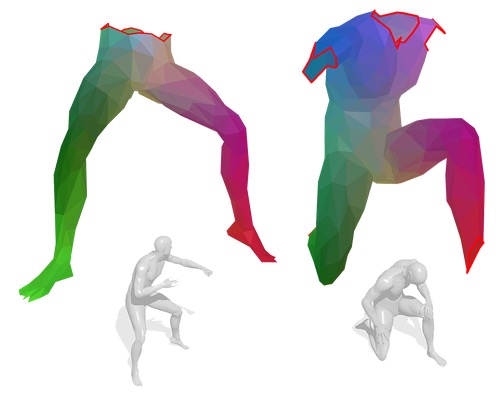}~\rotatebox{90}{\textcolor{gray}{\rule{3cm}{0.5pt}}}          \includegraphics[height=\teaserheight]{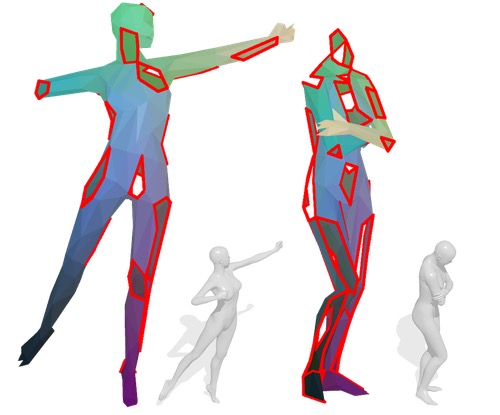} 
 & 
          \input{figures/winheuserVsOursTimeInHours} \\
          Orientation-preserving matching & Partial-to-partial shape matching (without using complete shape) 
          & \multicolumn{1}{c}{Runtime comparison to Windheuser et al.}
        \end{tabular}%
  }%
\captionof{figure}{We propose a \textbf{novel combinatorial solver for
the non-rigid matching of 3D shapes} based on discrete orientation-preserving diffeomorphisms~\cite{windheuser2011} (left). For the first time we utilize an orientation-preserving diffeomorphism to constrain the challenging problem of non-rigidly matching a pair of partial shapes without availability of complete shapes (center). Our solver scales significantly better compared to existing solvers and can thus handle shapes with practically relevant resolutions (right).
}
\label{fig:teaser}
\end{strip}

\input{abstract}

\input{introduction}

\input{relatedWork}
\input{background}

\input{problemSolving}

\input{numericalResults}

\section{Conclusion}
We proposed a novel combinatorial solver for efficiently computing solutions to the mathematically elegant integer linear programming formulation of 3D shape matching introduced in~\cite{windheuser2011}. Our solver consists of a primal-dual approach where the primal step makes use of min-marginals computed globally for the full problem. The original solver of \cite{windheuser2011} could only handle shapes of up to 250 triangles and therefore had to be applied in a heuristic coarse-to-fine strategy. In contrast, the proposed method leads to a drastic speedup and can therefore handle ILPs with millions of variables and 3D shapes of practically relevant resolution.
This allows us to globally optimize over all orientation preserving geometrically consistent non-rigid mappings providing state-of-the-art matching results.  Moreover, we extend this combinatorial matching approach to challenging scenarios like partial-to-partial shape matching without availability of the complete shapes.
We argue that also generalizations, e.g.\ to more than two shapes, should be practically solvable through a similar combinatorial optimization approach.

{\small
\bibliographystyle{ieee_fullname}
\bibliography{main}
}

\clearpage
\appendix
\input{appendix}

\end{document}

%% file: setup/acros.tex
\DeclareAcronym{LP}{
	short = LP ,
	long = Linear Program,
	short-plural = s,
	long-plural = s ,
}

\DeclareAcronym{ILP}{
	short = ILP ,
	long = Integer Linear Program,
	short-plural = s,
	long-plural = s ,
}

\DeclareAcronym{BP}{
	short = BP ,
	long = Binary Integer Linear Program,
	short-plural =  s,
	long-plural = s ,
}
\DeclareAcronym{LDD}{
	short = LDD ,
	long = Lagrangian Dual Decomposition,
	short-plural = s ,
	long-plural = s ,
}
\DeclareAcronym{ICP}{
	short = ICP ,
	long = Iterative Closed Point,
	short-plural = s ,
	long-plural = s ,
}
\DeclareAcronym{BnB}{
	short = BnB ,
	long = Branch and Bound,
	short-plural = s ,
	long-plural = s ,
}
\DeclareAcronym{FM}{
	short = FM ,
	long = Functional Maps,
	short-plural =  ,
	long-plural =  ,
}
\DeclareAcronym{QAP}{
	short = QAP ,
	long = Quadratic Assignment Problem,
	short-plural = s,
	long-plural = s,
}

%% file: figures/winheuserVsOursTimeInHours.tex
\definecolor{mycolor1}{rgb}{0.00000,0.44706,0.74118}%
\definecolor{mycolor2}{rgb}{0.85098,0.32549,0.09804}%

\begin{tikzpicture}%
\pgfplotsset{%
    width=6cm,
    height=3.9cm,
    label style = {font=\scriptsize},
    legend style = {font=\scriptsize},
    tick label style = {font=\scriptsize},
    every axis label = {font=\scriptsize},
    legend image code/.code={
    \draw[mark repeat=2,mark phase=2]
    plot coordinates {
    (0cm,0cm)
    (0.15cm,0cm)        %
    (0.3cm,0cm)         %
    };%
    }
}
\begin{axis}[%
xmin=100,
xmax=550,
xtick={100, 200, 300, 400, 500, 600, 700},
xlabel={\# Faces},
xlabel style={yshift=0.9em},
ymin=0,
ymax=3600,
ytick={1800, 3600},
yticklabels={0.5, 1},
ylabel={Time [h]},
ylabel style={yshift=-2.3em},
legend style={at={(0.9,0.98)},anchor=north east,draw=none}
]
\addplot [smooth, color=mycolor1, mark repeat=0.5, line width=1.5pt]
  table[row sep=crcr]{%
  100 -50\\
  100 -100\\
};
\addplot [smooth, color=mycolor2, line width=1.5pt]
  table[row sep=crcr]{%
  100 -50\\
  100 -100\\
};
\addlegendentry{Windheuser~\cite{windheuser2011}}
\addlegendentry{Ours}

\addplot [smooth, color=mycolor1,dashed, dash pattern=on 8pt off 8pt, line width=2.0pt]
  table[row sep=crcr]{%
100	140\\
125	297\\
150	574\\
200	2000\\
225	3345\\
230 3600\\
};

\addplot [smooth, dashed, dash pattern=on 8pt off 8pt, color=mycolor2, line width=2.0pt]
  table[row sep=crcr]{%
100	28\\
125	32\\
150	56\\
175	51\\
200	103\\
225	171\\
250	167\\
275	135\\
300	251\\
325	255\\
350	305\\
375	554\\
400	478\\
425	474\\
450	1528\\
475	2177\\
500	1441\\
525	2308\\
550	2128\\
};

\addplot [smooth, dotted, color=mycolor1, line width=2.0pt]
  table[row sep=crcr]{%
100	100\\
125	389\\
150	546\\
175	937\\
200	2189\\
225	3314\\
230 3600 \\
};

\addplot [smooth, dotted, color=mycolor2, line width=2.0pt]
  table[row sep=crcr]{%
100	22\\
125	15\\
150	26\\
175	50\\
200	61\\
225	75\\
250	113\\
275	213\\
300	266\\
325	301\\
350	299\\
375	334\\
400	379\\
425	349\\
450	1132\\
475	1825\\
500	1678\\
525	3024\\
550	2004\\
};

\end{axis}

\end{tikzpicture}%

%% file: abstract.tex
\begin{abstract}
We present a scalable combinatorial algorithm for globally optimizing over the space of geometrically consistent mappings between 3D shapes. We use the mathematically elegant formalism proposed by Windheuser et al.~\cite{windheuser2011} where 3D shape matching was formulated as an integer linear program over the space of orientation-preserving diffeomorphisms. Until now, the resulting formulation had limited practical applicability due to its complicated constraint structure and its large size. We propose a novel primal heuristic coupled with a Lagrange dual problem that is several orders of magnitudes faster compared to previous solvers. This allows us to handle shapes with substantially more triangles than previously solvable. We demonstrate compelling results on diverse datasets, and, even showcase that we can address the challenging setting of matching two partial shapes without availability of complete shapes. Our code is publicly available at \url{http://github.com/paul0noah/sm-comb}.
\end{abstract}

%% file: introduction.tex
\section{Introduction}
The shape matching problem is widely studied in graphics and vision due to its high relevance in numerous applications, including 3D reconstruction, tracking, shape modeling, shape retrieval, interpolation, texture transfer, or the canonicalisation of geometric data for deep learning. Shape matching refers to finding correspondences between parts of shapes -- for example, the shapes may be represented as triangular surface meshes, and correspondences may be obtained between vertices or triangles of individual shapes.
While identifying such correspondences is a relatively easy task for humans, from a computational perspective shape matching is much more challenging. This is because many formulations lead to high-dimensional combinatorial optimization problems.
While some of them are efficiently solvable, \eg based on the linear assignment problem~\cite{Munkres:1957ju}, such simple approaches do not take geometric relations between the shape parts into account and thus
typically lead to poor matchings.
In contrast, 
most practically relevant formulations account for geometric consistency in some form, which in turn lead to significantly more difficult optimization problems (e.g.~the NP-hard quadratic assignment problem~\cite{Pardalos:1993uo}, or mixed-integer programming formulations~\cite{bernard2020}). 

A decade ago Windheuser et al.~\cite{windheuser2011,windheuser2011a}  proposed an elegant formalism for the deformable matching of 3D shapes. Most notably, their approach is able to account for geometric consistency based on an orientation-preserving discrete diffeomorphism. While this formulation is conceptually appealing, it requires to solve a difficult integer linear programming (ILP) problem. The authors proposed to relax the binary variables to continuous variables, so that a linear programming (LP) problem is obtained. Although this LP formulation is convex and can thus be solved to global optimality, the resulting optimization problem is prohibitively large, and in turn only shapes with low resolution (at most 250 triangles) can be matched with existing solvers. 

In this work we close this gap and propose a scalable combinatorial solver for geometrically consistent deformable 3D shape matching. %
Our main contributions are:
\begin{itemize}
    \item We propose a novel primal heuristic, which, together with a Lagrange dual problem, gives rise to a combinatorial solver for orientation-preserving deformable 3D shape matching based on the formalism by Windheuser et al.~~\cite{windheuser2011,windheuser2011a}.
    \item 
    We show that in the special case of consistent triangulations between both shapes our primal heuristic can readily be used to solve the ILP formulation %
    to global optimally in polynomial time.
    \item Our solver is orders of magnitude faster than previous solvers, which allows us to handle shapes with a substantially higher resolution.
    \item Experimentally we demonstrate state-of-the-art results on a range of different shape matching problems, and we showcase the applicability to the difficult case of partial-to-partial shape matching.
\end{itemize}

%% file: relatedWork.tex
\section{Related Work}

In the following we give an overview of existing
shape matching approaches. For a more extensive overview we refer interested readers to the survey papers~\cite{vankaick2011, sahillioglu2020}.

\textbf{Rigid shape matching} methods consider the matching of shapes under the assumption that one shape undergoes a rigid-body transformation. %
For known correspondences, the resulting orthogonal Procrustes problem has a closed-form solution~\cite{schonemann1966}. In general, the assumption that correspondences are known is a severe limitation and typically one also wants to find the correspondences. Such tasks are commonly tackled via local optimization, \eg via the \ac{ICP} algorithm~\cite{besl1992} or variants thereof~\cite{estepar2004,billings2014, myronenko2010}.
 The main drawback of local methods is that they heavily depend on the initialization and generally do not obtain global optima. 
There are also global approaches, for example
based on semidefinite programming~\cite{maron2016}, \ac{BnB}~\cite{olsson2009}, or quasi-\ac{BnB}~\cite{dym2019}.
A major limitation of \emph{rigid} matching approaches is that in practice shapes are often non-rigidly deformed, for which rigid transformation models are too restricted.

For shapes that undergo an elastic transformation, \textbf{non-rigid shape matching} is more suitable.
One example %
is the functional map framework~\cite{ovsjanikov}, which
have led to remarkable results for isometric shapes~\cite{ovsjanikov, ren2019}. Yet, their main drawback is that they are generally sensitive to noise and less reliable in non-isometric settings. 
A wide variety of extensions of functional maps was developed to overcome specific drawbacks, e.g.~related to computation time~\cite{hu2021},  missing parts~\cite{rodola2017}, orientation preservation~\cite{ren2019}, deblurring~\cite{ezuz2017}, denoising~\cite{ren2019a} non-isometric shapes \cite{eisenberger2020}, or multi-matching~\cite{huang2020consistent,gao2021}.

An alternative non-rigid shape matching formulation is via the \textbf{\ac{QAP}}, also known as graph matching.
Due to the NP-hardness of the \ac{QAP}~\cite{Pardalos:1993uo}, numerous heuristic methods have been proposed~\cite{le-huu2017, leordeanu}. There also exist various convex relaxation methods for the \ac{QAP}~\cite{fogel2015,swoboda2017,dym2017,bernard2018,kushinsky2018}.
Although these methods are appealing as they can produce optimality bounds, they cannot guarantee to find global optima in all cases.
In \cite{bazaraa1979}, a globally optimal but exponential-time \ac{BnB} method was presented.
In \cite{vestner2017a}, the authors consider an iterative linearization of the \ac{QAP} for addressing shape matching. %
Recently,  simulated annealing was used for QAP-based shape matching~\cite{holzschuh2020}.
Similar to the \ac{QAP}, the ILP problem we address is NP-hard in general and theoretically solvable with (exponential-time) \ac{BnB} methods.

There are various alternative \textbf{local methods} for non-rigid shape matching, \eg based on Gromov-Hausdorff distances~\cite{memoli2005theoretical,bronstein2010gromov}, as-conformal-as-possible formulations~\cite{mandad2017}, iterative spectral upsampling~\cite{melzi2019}, a discrete solver for functional map energies~\cite{ren2021}, a hybrid spatial-spectral approach~\cite{xiang2021}, triangle-based deformation models~\cite{sumner2004},  elastic membrane energy optimization~\cite{ezuz2019}, and many more.
A major downside of local methods is that they heavily depend on a good initialization.

Opposed to local optimization approaches are \textbf{global methods}, which have the strong advantage that they are independent of the initialization.
Globally optimal non-rigid matchings can be obtained using shortest path algorithms for 2D shape matching~\cite{coughlan2000,felzenszwalb2005} and for 2D-to-3D matching~\cite{lahner2016}. 
Unfortunately, shortest path algorithms are not applicable for 3D-to-3D shape matching, since matchings cannot be represented as a shortest path but rather form a minimal surface~\cite{windheuser2011}. An alternative formulation based on a  convex relaxation was proposed in~\cite{chen2015}, which, however relies on an extrinsic term that requires a good spatial alignment.
In~\cite{bernard2020} a \ac{BnB} strategy is used to tackle a mixed-integer programming formulation that utilizes a  low-dimensional discrete matching model.
Despite the exponential worst-case complexity, it was shown that global optimality can be certified in most of the considered shape matching instances.
The framework by Windheuser et al.~\cite{windheuser2011, windheuser2011a, schmidt2014} considers an ILP formulation for the {orientation-preserving diffeomorphic matching} of 3D shapes. The main issue is that to date there is no efficient combinatorial solver that can solve large instances of respective problems, so that their framework is currently impracticable. The framework will be discussed in more detail in Sec.~\ref{sec:background}. The main objective of our work is to propose the first efficient solver tailored specifically to this formalism. %

\textbf{Learning-based shape matching}. A wide variety of learning-based approaches has been used to address  shape matching. While unsupervised approaches do not need time-consuming labeling of data~\cite{halimi2019,roufosse2019unsupervised,eisenberger2020a,donati2020,zeng2021,eisenberger2021}, supervised methods rely on the availability of labeled data~\cite{rodola2014,litany2017,charles2017,qi2017, groueix2019}. 
Overall, learning-based techniques are ideal to obtain task-specific features to solve shape matching problems, and it was demonstrated that respective approaches achieve remarkable results. However, a major difficulty is that they often lack the ability to generalize to other types of shapes, which in contrast is a major strength of learning-free approaches.
In this work we focus on a specific formalism in the latter class of optimization-based learning-free methods, and we believe that our work may be a first step towards integrating such powerful methods into modern learning approaches in order to eventually achieve the best from both worlds.

%% file: background.tex
\section{Background of the Shape Matching ILP}
\label{sec:background}
In the following we recapitulate the shape matching integer linear program (ILP)  of Windheuser et al.~\cite{windheuser2011}.
We provide the used notation in Tab.~\ref{table:notation}.
\newcommand{\energyVar}{\mathbbm{E}}
\begin{table}[h!]
\small
	\begin{tabularx}{\columnwidth}{ll}%
		\toprule
		$X,Y$ & shapes (triangular surface meshes)\\
		$(V_X,E_X,F_X) $ &  vertices, edges and triangles of shape $X$ \\
		\multirow{2}{*}{$F$} & product space of vertex-triangle,  \\ & edge-triangle and triangle-triangle pairs\\
		$E$ & edge product space \\
		$\Gamma \in \{0,1\}^{\abs{F}}$ & indicator vector of matches \\
		$\pi_X, \pi_Y$ & product vector projection on $X$ and $Y$ \\
		$\partial$ & geometric consistency constraints \\
		$\energyVar$ &  matching energy \\
		\bottomrule
	\end{tabularx}
	\caption{Notation used in this paper.}
	\label{table:notation}
\end{table}

\begin{figure}
    \centering
    \includegraphics[width=0.96\columnwidth]{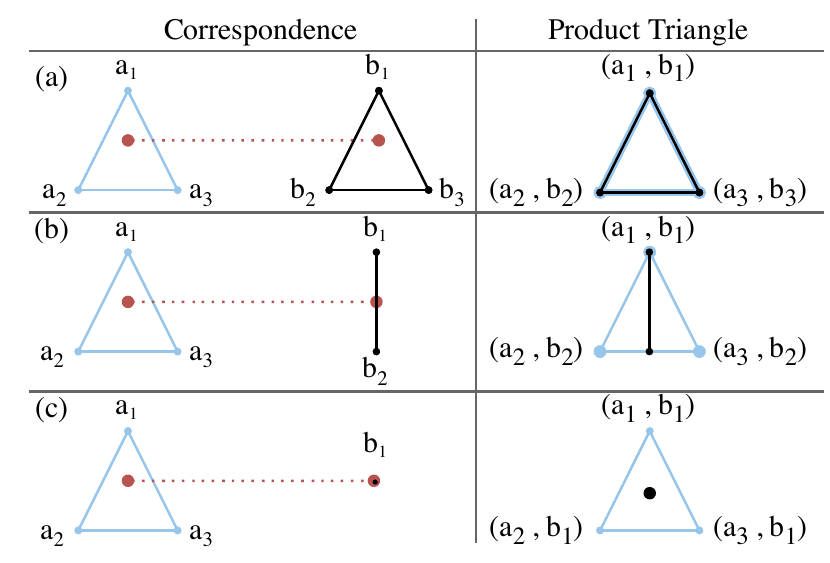}
    \caption{\textbf{Correspondences represented as product triangles} of $F$: If a product triangle (right) is part of the computed solution it implies that the respective triangle is matched to a triangle (top row), an edge (middle) or a vertex (bottom).  Each product triangle is associated with a local matching cost that encodes feature similarity and costs for stretching/compressing and bending.}
    \label{fig:correspondences}
\end{figure}

\begin{definition}[Shape]
We define a shape $X$ as a triplet $(V_X, E_X,F_X)$ of vertices $V_X$, 
edges 
$E_X \subset V_X \times V_X$ 
and triangles 
$F_X \subset V_X \times V_X \times V_X$,
such that the manifold induced by the triangles is oriented and has no boundaries\footnote{We note that partial shapes can be handled by closing holes and defining suitable costs for matching holes, see Appendix for more details.}.
\end{definition}
In the remainder we will consider two shapes $X = (V_X,E_X,F_X)$ and $Y=(V_Y,E_Y,F_Y)$ that shall be elastically matched to each other. A matching between $X$ and $Y$ is defined by  triangle-triangle, triangle-edge or triangle-vertex correspondences between pairs of elements in $X$ and $Y$.
We allow for triangle-edge and triangle-vertex matches to account for compressing and stretching shapes. For convenience, by $\overline{F}_\bullet$ we denote the set of \emph{degenerate} triangles, that in addition to the triangles $F_\bullet$ also contains triangles formed by edges (a triangle with two vertices at the same position) and triangles formed by vertices (a triangle with three vertices at the same position).
Similarly, we consider edge products and the set of degenerate edges $\overline{E}_\bullet$.
\begin{definition}[Product Spaces]
Let two shapes $X$ and $Y$ be given.
The triangle product space is defined as
\begin{equation}
	F \coloneqq \left\{ 
	\begin{pmatrix}
		a_1, b_1 \\
		a_2, b_2 \\ 
		a_3, b_3
	\end{pmatrix} \left|
	\begin{array}{ll}
	    (a_1 a_2 a_3 \in F_X \wedge b_1 b_2 b_3 \in \overline{F}_Y )~\vee \\
	     (a_1 a_2 a_3 \in \overline{F}_X \wedge b_1 b_2 b_3 \in {F}_Y ) 
	\end{array} \right.
	\right\}.\nonumber
\end{equation}
The edge product space is defined as
\begin{equation}
	E \coloneqq \left\{ 
	\begin{pmatrix}
		a_1, b_1 \\
		a_2, b_2
	\end{pmatrix} \left|
	\begin{array}{ll}
	    (a_1 a_2 \in E_X \wedge b_1 b_2 \in \overline{E}_Y )~\vee\\
	    (a_1 a_2 \in \overline{E}_X \wedge b_1 b_2 \in {E}_Y )
	\end{array} \right.
	\right\}.\nonumber
\end{equation}
\end{definition}
An illustration of possible correspondences represented as product triangles, \ie elements of the product triangle space, can be seen in Fig.~\ref{fig:correspondences}.
In order to guarantee a geometrically consistent matching, we impose two types of constraints:\\
(i) \textbf{Projection constraints $\pi$.}
We require that all triangles from $X$ must be matched to a vertex, edge or triangle from $Y$, and vice versa.
\begin{figure}[hbt!]
  \begin{minipage}[c]{0.3\columnwidth}
    \includegraphics[width=0.99\columnwidth]{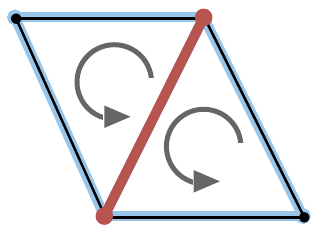}%
  \end{minipage}\hfill
  \begin{minipage}[c]{0.65\columnwidth}
    \caption{Two \textbf{product triangles} are neighboring if they share the same product edge with opposite orientation.}
    \label{fig:shared-product-edge}
  \end{minipage}
\end{figure} 
\\
\noindent(ii) \textbf{Geometric consistency constraints $\partial$.} Neighboring elements of $X$ must be matched to neighboring elements of $Y$.
Whenever a product edge is part of the matching, there must exist exactly two product triangles sharing the product edge oriented in opposite directions, see Fig.~\ref{fig:shared-product-edge}.
This ensures geometric consistency by requiring that the matching is a two-manifold in the product space, which also implies that the respective 3D mapping is orientation preserving.

Together with the energy $\energyVar \in \R^{\abs{F}}$  that quantifies the matching cost of individual correspondences (elements of the product space), the ILP shape matching problem reads
\begin{equation} 
    \underset{\Gamma \in \{0, 1\}^{|F|}}{\text{min}} \energyVar^\top  \Gamma  
    ~~  \text{s.t.} ~~
    \begin{pmatrix}
      \pi_X \\ \pi_Y \\ \partial
    \end{pmatrix}
    \Gamma
    = 
    \begin{pmatrix}
       \boldsymbol{1}_{|F_X|} \\ \boldsymbol{1}_{|F_Y|} \\ \boldsymbol{0}_{\abs{E}} \\
    \end{pmatrix},
\tag{ILP-SM}
\label{eq:opt-prob-discrete}
\end{equation}
where $\Gamma$ is the matching vector represented as indicator vector of the triangle product space $F$.
For more details about~\eqref{eq:opt-prob-discrete} we refer to \cite{windheuser2011, windheuser2011a, schmidt2014}.

\begin{figure*}[ht!]
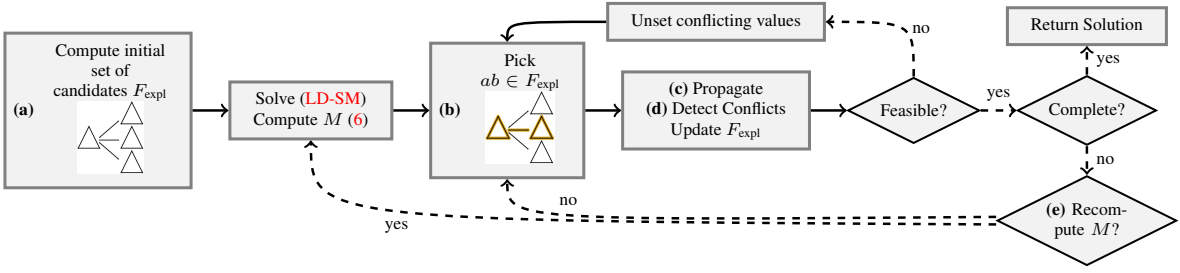

 	\centering
\hspace*{-2cm}
	     \includestandalone[width=\textwidth]{figures/solverPipeline}
		\caption{
		The \textbf{pipeline of our combinatorial solver} for the ILP shape matching problem~\eqref{eq:opt-prob-discrete}. The individual stages \textbf{(a)}-\textbf{(e)} are explained in Sec.~\ref{sec:primal-rounding}.
		}
		\label{fig:solver-pipeline}

\end{figure*}

\begin{prop}
If only triangle-triangle correspondences are allowed, problem~\eqref{eq:opt-prob-discrete} is solvable  to global optimality in polynomial time. 
\end{prop}
\begin{proof}
The two-manifold property of the matching and of both shapes $X$ and $Y$ yields that a single triangle-triangle match determines  all other triangle-triangle matches due to geometric consistency.
\end{proof}
Despite the conceptual elegance of the formulation, as soon as degenerate matchings are allowed for (\eg triangle-edge or triangle-vertex), which is necessary for virtually any real-world shape matching instance,  problem~\eqref{eq:opt-prob-discrete} is significantly more difficult. Specifically, it belongs to the class of ILPs, which are NP-hard in general. Windheuser et al.~\cite{windheuser2011, windheuser2011a, schmidt2014} address this based on an LP relaxation that relaxes the binary variables to continuous ones. However, even their convex LP formulation has several drawbacks that impede a practical application:\\
    (i) The relaxed LP formulation involves a \emph{very large number of variables}. For example, non-rigidly matching two 3D shapes with 1000 faces each leads to a total of about $2\cdot 10^7$ variables.  The authors implement an efficient parallelized GPU-based primal-dual solver~\cite{eckstein1990alternating}, which can handle problems with up to 250 faces (leading to about $10^6$ binary variables), requiring  a total time of about 2\,h. The same applies to modern state-of-the-art LP solvers~\cite{gurobi}. Currently no solver exists that can solve even moderately-sized instances of problem~\eqref{eq:opt-prob-discrete}. \\
 (ii)  Windheuser et al.~attempt to address this based on a coarse-to-fine scheme. Initially, at the coarsest scale, they match severely downsampled shapes, and repeatedly apply their framework to refine the non-rigid matching only in the local neighborhood of the matching at the previous coarser scale. Overall, this has the downside that the final matching substantially depends on the initial matching at the coarsest scale. With that, there is the risk that the initial low resolution shapes do not contain sufficient details for finding a reliable matching. \\
(iii) Furthermore, solutions of the continuous relaxation are generally not  binary, so that a rounding step is necessary to obtain a discrete solution. To this end, the authors propose to repeatedly solve the expensive LP relaxations while gradually fixing more and more variables to be binary.

Overall, to make problem~\eqref{eq:opt-prob-discrete} practicable, being able to solve larger shape matching instances is of great importance.
To address this, we 
propose an efficient combinatorial solver based on the Lagrange decomposition to optimize problem~\eqref{eq:opt-prob-discrete}.

\section{Lagrange Decomposition}
 Next,  we introduce our Lagrange decomposition reformulation for problem~\eqref{eq:opt-prob-discrete}, which is amenable to dual optimization.

\paragraph{Decomposition into subproblems.}
We associate a small subproblem $\mathcal{S}$ for each row of the constraint matrix in~\eqref{eq:opt-prob-discrete}:

\begin{definition}[Subproblems]
For each individual projection and geometry consistency constraint we define a set of subproblems as
\begin{align}
\begin{multlined}[t][0.8\columnwidth]%
\forall j:~\mathcal{S}^{\pi_X}_j = \{\Gamma \in \{0,1\}^{\abs{F}} : \sum_{i} (\pi_X)_{ji} \Gamma_i = 1\} \,,
\end{multlined}
\\
\begin{multlined}[t][0.8\columnwidth]%
\forall j:~\mathcal{S}^{\pi_Y}_j = \{\Gamma \in \{0,1\}^{\abs{F}} : \sum_{i} (\pi_Y)_{ji} \Gamma_i = 1\}\,,
\end{multlined}
\\
\begin{multlined}[t][0.8\columnwidth]%
\forall j:~\mathcal{S}^{\partial}_j = \{ \Gamma \in \{0,1\}^{\abs{F}} : \sum_{i} (\partial)_{ji} \Gamma_i = 0 \}\,.
    \end{multlined}
\end{align}
We write the set of all subproblems as 
\begin{equation}
   \bm{\mathcal{S}} =  \{
   \mathcal{S}^{\pi_X}_{j=1,\ldots,\abs{F_X}},
   \mathcal{S}^{\pi_Y}_{j=1,\ldots,\abs{F_Y}},
   \mathcal{S}^{\partial}_{j=1,\ldots,\abs{E}}
   \}.
\end{equation}
\end{definition}
With the above decomposition we can write the Lagrange dual shape matching problem as
\begin{equation}
\tag{LD-SM}
\begin{array}{rl}
    \max\limits_{\boldsymbol{\lambda} = \{ \lambda^{\mathcal{S}}\}} & 
    \sum\limits_{\mathcal{S} \in \bm{\mathcal{S}}}
     ~\min\limits_{\Gamma \in \mathcal{S}} \la \lambda^{\mathcal{S}}, \Gamma \ra \\
    \text{s.t.} &
    \sum\limits_{\mathcal{S} \in \bm{\mathcal{S}}} \lambda^{\mathcal{S}} = \energyVar\,.
    \end{array}
    \label{eq:dual-shape-matching}
\end{equation}

\paragraph{Min-marginals.}
While we cannot easily derive a matching from a solution $\boldsymbol{\lambda}$ of the dual problem~\eqref{eq:dual-shape-matching}, we can nevertheless obtain important dual costs that will guide our primal solution search.
We use dual costs based on min-marginals, which are defined as follows:

\begin{definition}[Min-marginals]
For any subproblem $\mathcal{S} \in \bm{\mathcal{S}}$ we define the min-marginal for the $i$-th variable as the difference of the optima with the corresponding variable fixed to $1$ vs.\ $0$ as 
\begin{equation}
m^{\mathcal{S}}_{i} = 
\min_{\Gamma \in \mathcal{S}: \Gamma_i = 1} \la \lambda^\mathcal{S}, \Gamma \ra - 
\min_{\Gamma \in \mathcal{S}: \Gamma_i = 0} \la \lambda^\mathcal{S}, \Gamma \ra \,.
\end{equation}
In words, a min-marginal quantifies by how much a variable wants to attain the value $1$ resp.\ 0 in the subproblem $\mathcal{S}$.

The total min-marginal is defined as the sum over all min-marginals
\begin{equation}
\label{eq:total-min-marginals}
    M_i = \sum_{\mathcal{S} \in \bm{\mathcal{S}}} m^{\mathcal{S}}_i\,.
\end{equation}
\end{definition}

If for each variable all min-marginals have the same sign, and the total min-marginal is non-zero, we can directly reconstruct a primal solution by setting $\Gamma_i = 1$ if $M_i < 0$ $\Gamma_i = 0$ if $M_i > 0$.
This case occurs when the relaxation defined by the Lagrange decomposition~\eqref{eq:dual-shape-matching} is tight, which is not true in general.
If not tight, the above reconstruction will result in infeasible $\Gamma$.
However, in that case  good solutions will mostly agree with the sign of the total min-marginal, which we exploit in our primal rounding strategy.
We optimize the Lagrange decomposition and compute min-marginals with the approximate solver~\cite{lange2021}.

%% file: figures/solverPipeline.tex
\begin{tikzpicture}[
	roundnode/.style={diamond, aspect=2, draw=black!100, fill=gray!10, thick, text width=1.5cm, align=center, inner sep=-0cm},
	squarednode/.style={rectangle, draw=gray!100, fill=gray!10, very thick, minimum size=3mm},
	node distance=0.4cm and 0.5cm
	]
	\tikzstyle{every node}=[font=\scriptsize]
	
	\node[squarednode]   (init)       {\textbf{(a)} \begin{tabular}{c} Compute initial\\  set of \\candidates $F_{\text{expl}}$ 
	\\\includegraphics[height=1cm]{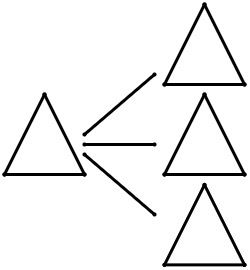}\end{tabular}};
	
	\node[squarednode]   (dualsolver)       [right=of init] {
	\begin{tabular}{c}
	    Solve \eqref{eq:dual-shape-matching} \\
	    Compute $M$~\eqref{eq:total-min-marginals}
	\end{tabular}};
	
	\node[squarednode]   (pick)       [right=of dualsolver] { \textbf{(b)}
	\begin{tabular}{c}
	    Pick \\$ab \in F_\text{expl}$ \\
	    \includegraphics[height=1cm]{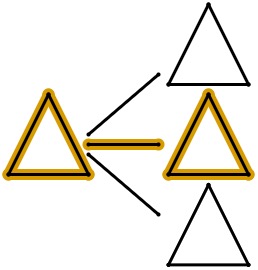}
	\end{tabular}};
	
	\node[squarednode]   (propagate)       [right=of pick] {
	\begin{tabular}{c}
	    \textbf{(c)} Propagate\\ \textbf{(d)} Detect Conflicts \\ Update $F_\text{expl}$
	\end{tabular}};
	
    \node[roundnode]   (feasible)       [right=of propagate] { Feasible? };
	
	\node[roundnode]   (complete)       [right=of feasible] { Complete? };
	
	\node[roundnode]   (recompmm)       [below=of complete] { \textbf{(e)} Recompute $M$? };
	
	\node[squarednode]   (backtrack)       [above=of propagate] {
	\begin{tabular}{c}
	    Unset conflicting values
	\end{tabular}
	};
	
	\node[squarednode]   (return)       [above=of complete] {
	\begin{tabular}{c}
	    Return Solution
	\end{tabular}};

	\draw[line width=0.35mm, ->] (init.east) -- (dualsolver.west);
	\draw[line width=0.35mm, ->] (dualsolver.east) -- (pick.west);
	\draw[line width=0.35mm, ->] (pick.east) -- (propagate.west);
	\draw[line width=0.35mm, ->] (propagate.east) -- (feasible.west);
	\draw[dashed, line width=0.35mm, ->] (feasible.east) -- (complete.west) node [midway, above] {yes};
	
	\path[name path=feasible-north] (feasible) -- +(0,1.32);
	\path[name path=backtrack-east] (backtrack) -- +(3,0);
	\path [name intersections = {of = feasible-north and backtrack-east, by = feasible-backtrack}]; 
	\draw[dashed, line width=0.35mm, ->] (feasible.north) .. controls (feasible-backtrack) .. (backtrack) node [midway, right] {no};
	
	\draw[dashed,line width=0.35mm, ->] (complete.north) -- (return.south) node [midway, right] {yes};
	\draw[dashed,line width=0.35mm, ->] (complete.south) -- (recompmm.north) node [midway, right] {no};
	
	\path[name path=backtrack-west] (backtrack) -- +(-4,0);
	\path[name path=pick-north] (pick) -- +(0,+1.3);
	\path [name intersections = {of = backtrack-west and pick-north, by = backtrack-pick}]; 
	\draw[line width=0.35mm, ->] (backtrack.west) .. controls (backtrack-pick) .. (pick);

	\path[name path=recompmm-east] (recompmm) -- +(-17,0);
	\path[name path=pick-south] (pick) -- +(0,-2.2);
	\path [name intersections = {of = recompmm-east and pick-south, by = recompmm-pick}]; 
	\draw[dashed, line width=0.35mm, ->] ([yshift=+1.5pt]recompmm.west) .. controls (recompmm-pick) .. (pick) node [midway, above] {no};
	
	\path[name path=dualsolver-south] (dualsolver) -- +(0,-2.2);
	\path [name intersections = {of = recompmm-east and dualsolver-south, by = recompmm-dualsolver}]; 
	\draw[dashed, line width=0.35mm, ->] ([yshift=-1.5pt]recompmm.west) .. controls (recompmm-dualsolver) .. (dualsolver) node [midway, below] {yes};
	
\end{tikzpicture}

%% file: problemSolving.tex
\begin{table*}[h!t!]
\footnotesize
	\begin{tabular}{c | c | c}
		\begin{minipage}[t]{.3\textwidth}
		\begin{center}\textbf{Injectivity}\end{center}%
		\vspace{-7mm}
			\begin{center}%
			\includegraphics[height=2.5cm]{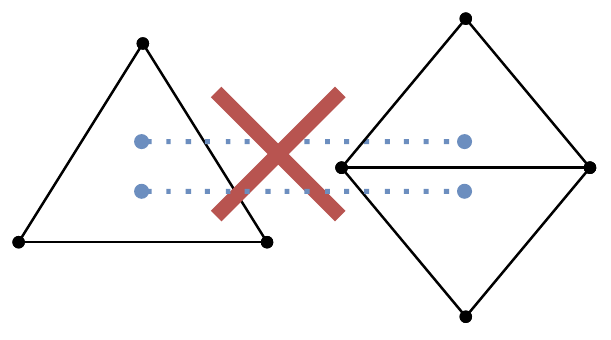}%
			\end{center}%
				\vspace{-4mm}
			 \textbf{Constraint:} Each face can only have one matching. \\
			 \textbf{Propagation:} 
			 Matching a triangle from one shape implies all other correspondences involving the same triangle have to be set to $0$.
		\end{minipage}
		&
		\begin{minipage}[t]{.3\textwidth}
		\begin{center}\textbf{Surjectivity}\end{center}
		\vspace{-7mm}
			\begin{center}
			\includegraphics[height=2.5cm]{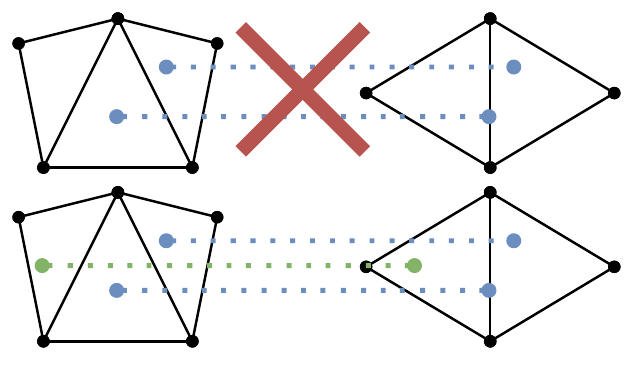}
			\end{center}
			\vspace{-4mm}
			\textbf{Constraint:} Each face of each shape needs to have a matching.\\
			\textbf{Propagation:}
			 If all correspondences involving a given triangle are $0$, then the remaining one is set to $1$.
		\end{minipage}
		&
		\begin{minipage}[t]{.3\textwidth}
		    \begin{center}\textbf{Two-manifold}\end{center}
		    \vspace{-7mm}
			\begin{center}
				\includegraphics[height=2.5cm]{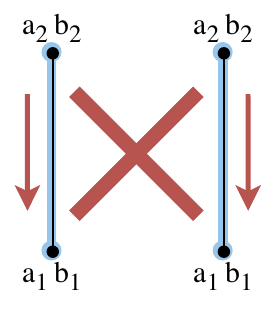}
			\end{center}
			\vspace{-4mm}
			\textbf{Constraint:} The same product edge orientation can only occur once
			in the solution\\
			\textbf{Propagation:} If a product edge is part of a matching, all other occurrences of the same product edge are set to $0$.\vspace{1em}
		\end{minipage}\\
		\hline
		\begin{minipage}[t]{.3\textwidth}
		\begin{center}\textbf{Geometric Consistency I}\end{center}
		\vspace{-7mm}
		\begin{center}
		\includegraphics[height=2.4cm]{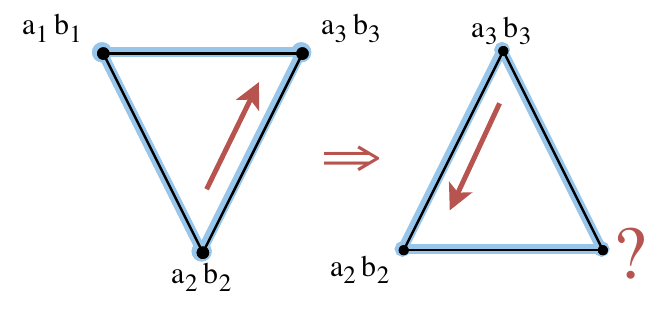}
		\end{center}
		\vspace{-4mm}
		\textbf{Constraint:} Each product edge needs to have a counterpart (opposite orientation). \\
		\textbf{Propagation:} A solution must include exactly one of the product triangles containing the opposite orientation of the product edge.
		\end{minipage}
		&
		\begin{minipage}[t]{.3\textwidth}
		\begin{center}\textbf{Geometric Consistency II}\end{center}
		\vspace{-7mm}
		\begin{center}
		\includegraphics[height=2.35cm]{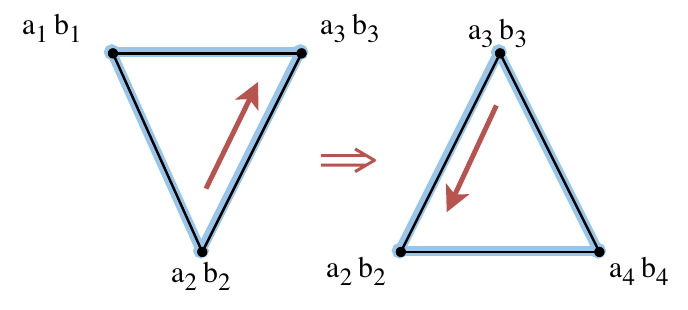}
		\end{center}
		\vspace{-4mm}
		\textbf{Constraint:}
		Each product edge needs to have a counterpart (opposite orientation). 
		\\
		\textbf{Propagation:} If there is only one counterpart left within the possible matches, it has to be part of the solution.
		\end{minipage}
		&
		\begin{minipage}[t]{.3\textwidth}
		    \begin{center}\textbf{Closedness}\end{center}
		    \vspace{-7mm}
			\begin{center}
			\includegraphics[height=2.4cm]{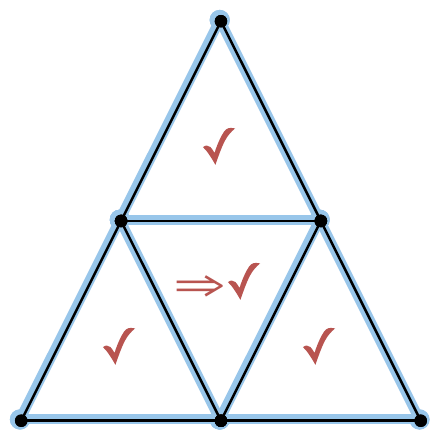}
			\end{center}
			\vspace{-4mm}
			\textbf{Constraint:} No holes are allowed in the matching. \\
			\textbf{Propagation:} Whenever all three neighbors of a product face are part of the solution, the product face itself is also part of the solution.
		\end{minipage}\\
	\end{tabular}
	\caption{\textbf{Matching constraints and the derived propagation rules} used in the primal heuristic. Product triangles are visualized with blue-black edges, triangles of 3D shapes with black edges.}
	\label{tab:propagation-rules}
\end{table*}

\section{Primal Rounding}
\label{sec:primal-rounding}
In the following we introduce our heuristic for primal rounding, i.e.\ obtaining primal solutions based on the dual costs that were computed through the Lagrange decomposition~\eqref{eq:dual-shape-matching}. Before we explain details, we provide a high-level summary of the main concept:
First, we pick a suitable initial product triangle as first correspondence candidate.
After we have solved the Lagrange decomposition and obtained the total min-marginals, we iteratively add product triangle matchings, i.e.~we select elements of the matching vector $\Gamma$ that are to be set to $1$. This may induce additional product triangle correspondences (due to the constraints in problem~\eqref{eq:opt-prob-discrete}), so we also force assignments of other variables. In case a conflict arises, we detect it and backtrack.
After a given number of variable assignments, or if too many backtracking steps occur, we solve the Lagrange decomposition~\eqref{eq:dual-shape-matching} again, while fixing the already found correspondences.
With the updated total min-marginals we start the search again until we find a complete solution. Our overall pipeline is shown in Fig.~\ref{fig:solver-pipeline}, which we explain next in detail.

\noindent\textbf{(a) Initialization.}
Given the empty matching, we choose the first triangle-triangle matching as follows:\\
(i) Consider the shape $Z \in \{X,Y\}$ with fewer triangles.\\
(ii) In $Z$, we choose the most regular triangle $z$ (all angles between $20^\circ$ and $90^\circ$, area is close to the mean area of all triangles, triangle lies in low curvature region).\\
(iii) We select all elements in $\Gamma$ that form non-degenerate triangle-triangle matchings of $z$. \\ %
(iv)  Among these, we choose the matching candidate with smallest total min-marginals.

\noindent\textbf{(b) Exploration of candidates.}
We maintain a set of individual triangle-triangle matchings,~i.e.~elements of the matching vector $\Gamma$, to explore in subsequent iterations $F_{\text{expl}} \subset F$.
After adding a product triangle $ab = (a_1 b_1, a_2 b_2, a_3 b_3) \in F_{\text{expl}}$ to our matching $\Gamma$, we add all other product triangles $a'b' \in F$ that share a product edge oriented in the opposite direction with the currently selected $ab$ to $F_{\text{expl}}$.
In other words, in the next iteration we select one of the tentative product triangles such that the \textbf{Geometric Consistency I} constraint (Tab.~\ref{tab:propagation-rules}) is fulfilled for one of the already selected product triangles.
We explore product triangles in $F_{\text{expl}}$ according to their total min-marginals in ascending order (see Appendix for a formal description).

Overall, the purpose of exploring only the elements of $F_\text{expl}$ minimizes the possibility of obtaining several disjoint submatchings that do not fit together geometrically.

\noindent\textbf{(c) Constraints \& propagation.}
After each individual variable assignment we analyze all constraints that involve the modified variable, and then check whether any other variable assignments are forced. 
If so, we set the forced variable and recursively propagate.
All used constraints and propagation rules are summarized shown in Tab.~\ref{tab:propagation-rules}.

\noindent\textbf{(d) Conflict detection.}
After each variable assignment we check for potential conflicts due to
two potential cases:\\
(i) \textit{Infeasibility}: individual constraints are not satisfiable anymore by the variables that are not set yet.\\
(ii) \textit{Contradicting assignments}: variables which are already set to one would have to be set to zero by the propagation and vice versa.

In cases of conflicts we perform backtracking by undoing the assignments and respective propagations that participated in the conflict.

\noindent\textbf{(e) Recomputation of min-marginals.}
In order to better reflect the quality of product triangle candidates in $F_{\text{expl}}$ w.r.t.\ already obtained partial matching, we regularly recompute (total) min-marginals. 
Whenever a certain amount of product triangles is set, we fix the respective variables in~\eqref{eq:opt-prob-discrete}, dualize the subproblem to obtain a new reduced Lagrange decomposition~\eqref{eq:dual-shape-matching}, and eventually optimize again to obtain the updated total min-marginals $M_i$.
We refer to the Appendix for details. %

%% file: numericalResults.tex
\begin{figure*}[ht!]
	\begin{center}
		\includegraphics[width=\linewidth]{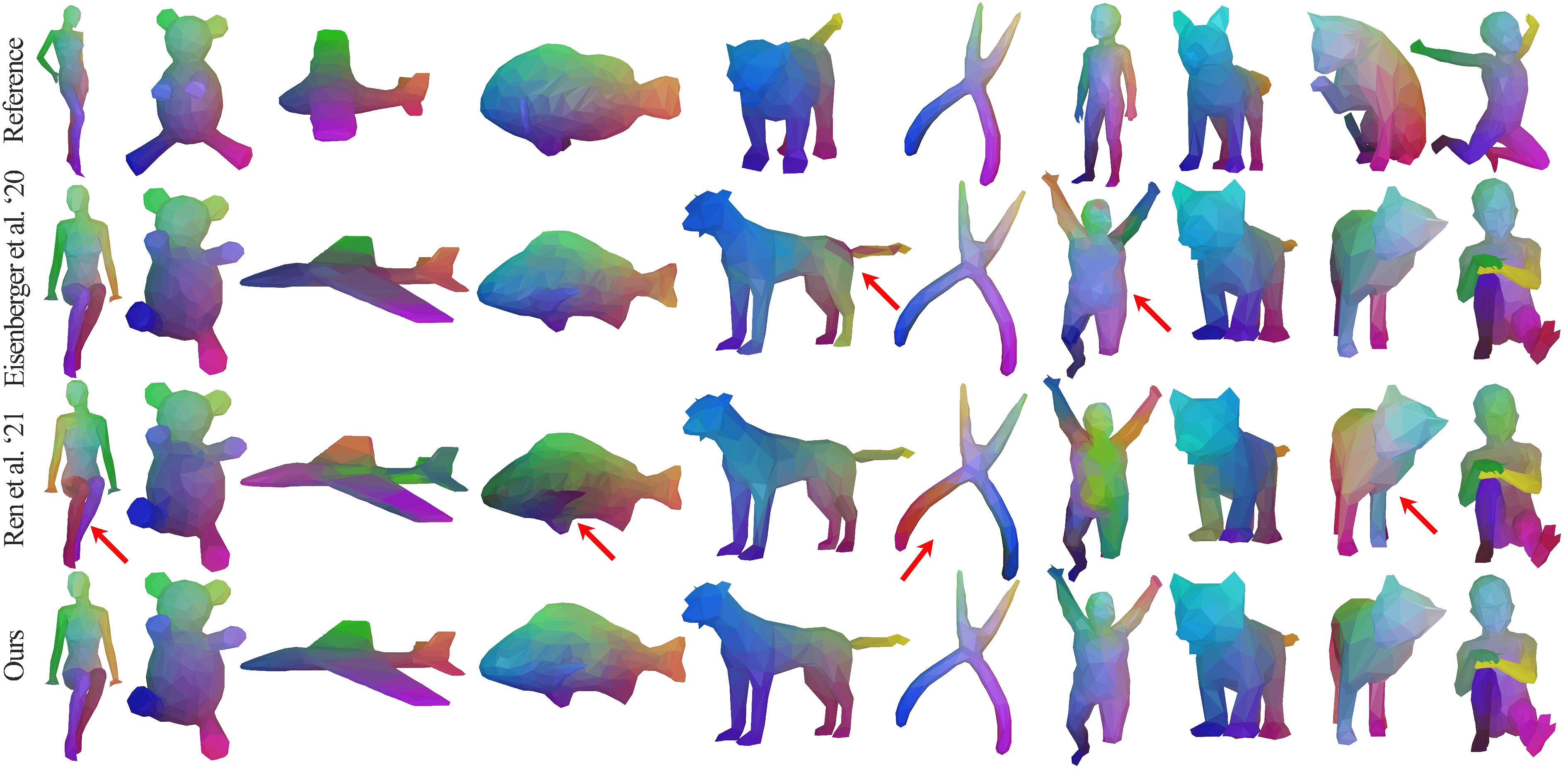}
		\caption{\textbf{Qualitative comparison} of the method by Eisenberger et al.~\cite{eisenberger2020} (second row), Ren et al.~\cite{ren2021} (third row) and Ours (last row). While Eisenberger et al.~and Ren et al.~do not guarantee orientation preservation they lead to erroneous matchings (\eg left-right flips, see red arrows), whereas our method leads to smooth and orientation-preserving matchings.
		}
		\label{fig:against-sota-qualitative}
	\end{center}
\end{figure*}

\begin{figure}[h!]
	\begin{center}
	\footnotesize
	\begin{tabular}{cc} \\
	\setlength\tabcolsep{0pt}
	     \hspace{-0.7cm}\input{figures/pckVsWindAvg}\hspace{-0.4cm}
	     & \begin{tabular}{cc}%
	         \multicolumn{2}{c}{\includegraphics[height=1.5cm]{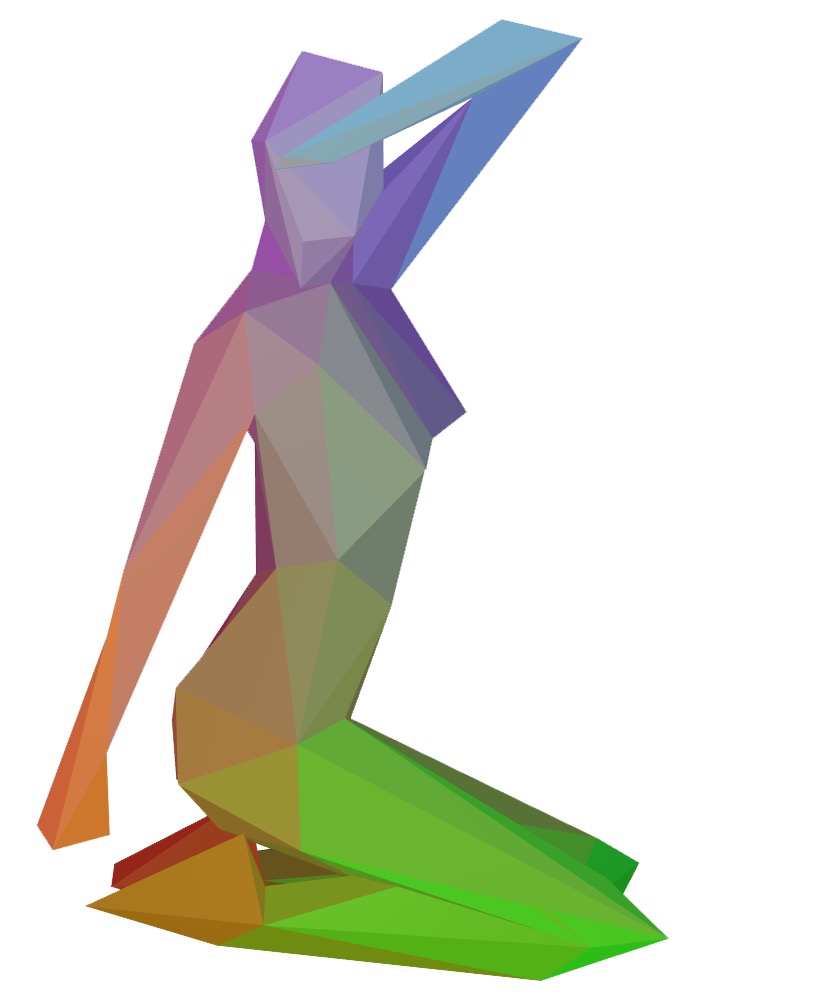}}\\
	          \multicolumn{2}{c}{Reference}\\
	          \includegraphics[height=1.55cm]{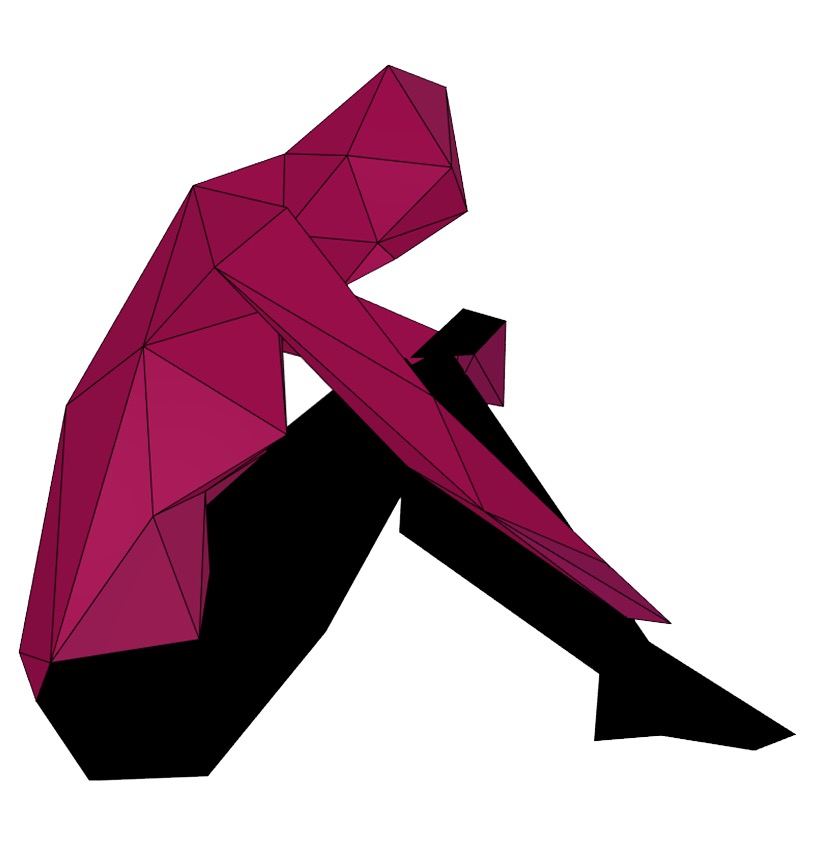}&%
	          \includegraphics[height=1.55cm]{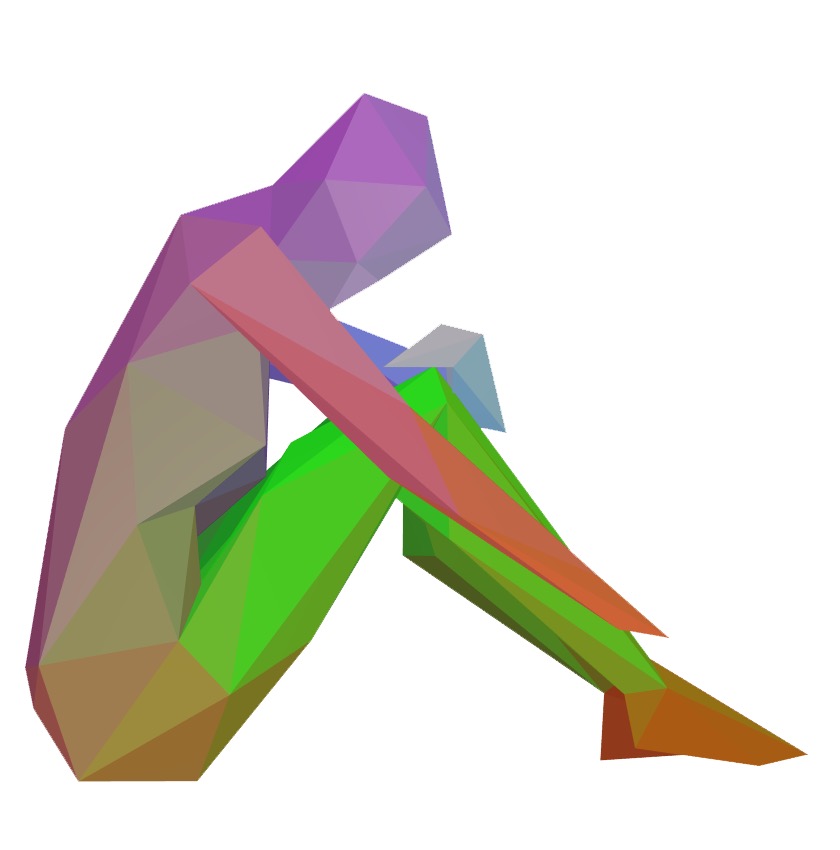}\\%
	           Windheuser & Ours
	     \end{tabular}
	\end{tabular}
		\caption{Comparison of the average percentage of correct matchings for the entire TOSCA dataset  of \textbf{Windheuser et al.~vs.~Ours} (left). The horizontal axis shows the geodesic error threshold, and the vertical axis shows the percentage of matches that are smaller than or equal to this error. 
		For Windheuser et al.~we allow the solver to take $10\times$ more time than our method needed (per shape matching instance) -- even then the curve of Windheuser et al.~is low because it is unable to find good matchings within the given time budget, see the qualitative example on the right (black shows unmatched parts, all shapes have $175$ triangles).}
		\label{fig:toscaoursvswindheuser}
	\end{center}
\end{figure}

\section{Experiments}
\label{sec:experiments}
In the following we experimentally evaluate our approach on various datasets in a range of different settings.

\paragraph{Shape matching data.} In our experiments we consider shape matching instances from several datasets: TOSCA~\cite{bronstein2008}, TOSCA partial~\cite{rodola2017}, SHREC-watertight~\cite{giorgi2007}, SMAL~\cite{zuffi2017}, SHREC~'19~\cite{melzi2019shrec} and KIDS~\cite{rodola2014}.
We downsample all meshes to about at most $1000$ faces.
We do not perform post-processing on the obtained matchings.
The energy $\mathbb{E}$ of problem~\eqref{eq:opt-prob-discrete} is computed analogously to~\cite{windheuser2011a}.

\paragraph{Shape matching algorithms.}
Since our main objective is to improve the computational performance of the best existing solver for problem~\eqref{eq:opt-prob-discrete}, as a baseline we reimplemented the rounding strategy 
proposed by Windheuser et al.~\cite{windheuser2011}\footnote{The original code is not available.} based on the state-of-the-art LP solver Gurobi~\cite{gurobi}. For further details we refer to the Appendix.

In addition, we also compare our solver for problem~\eqref{eq:opt-prob-discrete} with two recent
state-of-the-art methods that rely on other shape matching formalisms. Among them is a method  based on smooth shells (Eisenberger et al.)~\cite{eisenberger2020}, and a method based on a discrete functional map optimization framework (Ren et al.)~\cite{ren2021}.

\subsection{Combinatorial Solvers for Problem~\eqref{eq:opt-prob-discrete}}

First, we compare against the directly related approach of Windheuser et al.~\cite{windheuser2011}, which solves the same problem~\eqref{eq:opt-prob-discrete} as ours. 

In Fig.~\ref{fig:teaser} (right) we show the scalability of the solver of  Windheuser et al.~and ours depending on the number of triangles per shape. We find that while Windheuser et al.~already takes 1\, h for low-resolution shapes with  $\approx200$ triangles (leading to a total of $\approx\,8 \cdot 10^{5}$ binary variables in problem~\eqref{eq:opt-prob-discrete}),
our method scales significantly better and can handle shapes with substantially higher resolutions.
Our method has a linear memory consumption (to the problem size, which is quadratic in the shape resolution).
We note that the bump in the graph stems from the heuristically determined recomputation of the min-marginals which may vary for individual matching problems (see Sec. A3 in Appendix).

In Fig.~\ref{fig:toscaoursvswindheuser} we show quantitative and qualitative results of both solvers on the full TOSCA dataset, where we have found that our method performs significantly better with an average area  under the curve (higher is better) of $0.91$ vs.~$0.72$ for Windheuser et al. (see Appendix for more details).

\newcommand{\partialheight}[0]{2.78cm}
\begin{figure*}[ht!]
 \centerline{%
  \footnotesize%
\setlength\tabcolsep{0pt}
  \begin{tabular}{cccccc}%
          \includegraphics[height=\partialheight]{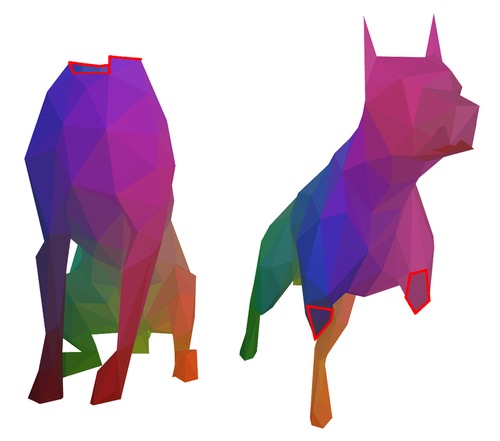}%
          &%
          \includegraphics[height=\partialheight]{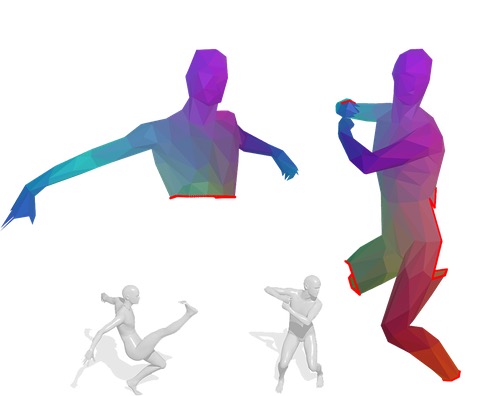}%
          &%
          \includegraphics[height=\partialheight]{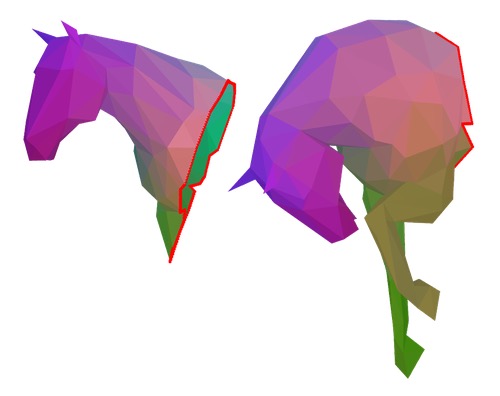}%
          &%
          \includegraphics[height=\partialheight]{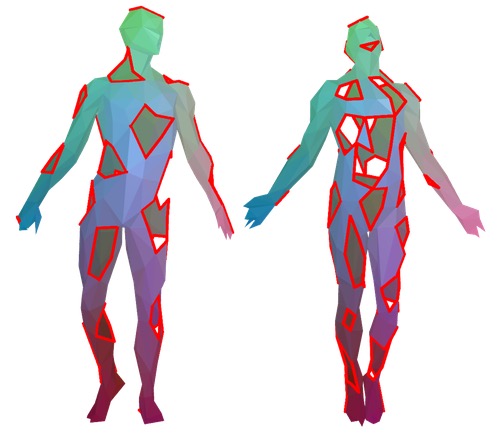}%
          &%
          \includegraphics[height=\partialheight]{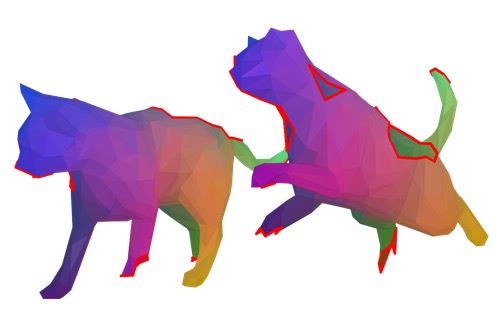}%
        \end{tabular}%
  }%
\caption{Our method can handle the difficult case of \textbf{matching pairs of partial shapes without availability of complete shapes}, in which the orientation-preserving diffeomorphism serves as powerful constraint (boundaries are shown as red lines).
(Best viewed magnified on screen)
}
\label{fig:partial}
\end{figure*}

\begin{figure}[h!]
	\begin{center}
	    \input{figures/pckCombined}
		\caption{Comparison of the percentage of correct matchings for the \textbf{TOSCA dataset} (left), the \textbf{KIDS dataset} (middle) and the \textbf{SHREC'19 dataset} (right). 
		The horizontal axis shows the geodesic error threshold, and the vertical axis shows the percentage of matches smaller than or equal to this error. 
		Ours obtains the highest area under curve ($\uparrow$)  ($\textbf{0.929}$ for ours vs.~$0.836$ for Ren et al.~and $0.910$ for Eisenberger et al.~on TOSCA).}
		\label{fig:matching-accuracy-vs-sota}
	\end{center}
\end{figure}
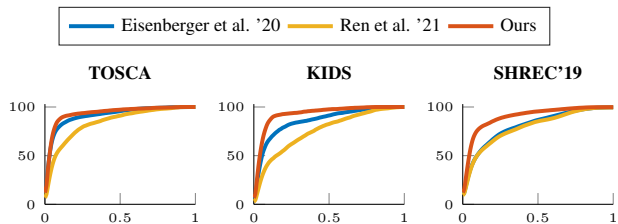

\subsection{Comparison to State of the Art}
Next, we compare our method to recent state-of-the-art shape matching approaches. Since Eisenberger et al.~\cite{eisenberger2020} show that their recent method substantially outperforms a range of other methods (BCICP \cite{ren2019}, Zoomout \cite{melzi2019}, KM \cite{vestner2017a}, FM \cite{ovsjanikov}, BIM \cite{kim2011}) on various datasets, for our comparison we only focus on the method by Eisenberger et al.~\cite{eisenberger2020}, and in addition also include results of the more recent method of Ren et al.~\cite{ren2021}.

In Fig.~\ref{fig:against-sota-qualitative} we show qualitative results for various non-rigid 3D shape matching instances from the datasets TOSCA, SHREC-watertight, SMAL and KIDS. The method of Ren et al.~suffers from matchings that are not geometrically consistent, which thus leads to nonsmooth matchings (\eg the pliers in the sixth column, or the kid in the eigth column). Moreover, both Ren et al.~and Eisenberger et al.~do not guarantee orientation preservation, so that for example left-right flips occur (\eg the lion in the fifth column for Eisenberger et al., or the cat in the second-last column for Ren et al.). In contrast, our method obtains reliable matchings in these cases which are smooth and preserve the orientation.

In Fig.~\ref{fig:matching-accuracy-vs-sota} we show quantitative results on the TOSCA, KIDS and SHREC'19 datasets.
Despite the fact that our method aims to directly solve a high-dimensional ILP with up to $2\cdot10^7$ binary variables (see problem~\eqref{eq:opt-prob-discrete}), our method outperforms the baselines in terms of solution quality. Due to the large number of binary variables, our approach requires $22.68$ min on average to compute a matching (opposed to few seconds for Ren et al.~and Eisenberger et al.).
For some individual shape classes we have found that ours has slightly worse performance (see Appendix), which stems from  poor local optima for individual shape matching instances.

\subsection{Partial-to-Partial Non-Rigid Matching}
In Fig.~\ref{fig:partial} we showcase that our proposed solver is even able to handle difficult non-rigid partial-to-partial shape matching problems. Although partial shape matching has attracted a lot of attention recently~\cite{litany2016non,rodola2017,litany2017,vestner2017a,melzi2019,bernard2020,xiang2021}, most existing works typically consider the case of matching a partial shape to a complete shape. There are also some partial-partial shape matching methods \cite{comso2016,attaiki2021dpfm} that build upon machine learning, which we consider to be
orthogonal to our method. 
In contrast, for the first time we utilize orientation-preserving diffeomorphisms to constrain the challenging problem of non-rigidly matching a pair of partial shapes {without availability of complete shapes}.

\section{Discussion \& Limitations}
Our experimental analysis has confirmed that we can compute matchings involving up to $2 \cdot 10^7$ variables on a range of datasets within about $2$h.
Although our solver does not guarantee to find globally optimal solutions and thus may lead to suboptimal results in some cases (see Appendix), our experiments confirm that in most cases we produce high-quality matchings, and that we can even handle partial-to-partial shape matching. 

%% file: figures/pckVsWindAvg.tex
\definecolor{mycolor1}{rgb}{0.00000,0.44706,0.74118}%
\definecolor{mycolor2}{rgb}{0.85098,0.32549,0.09804}%

\newcommand{\subplotwidth}{0.24\columnwidth}
\newcommand{\subplotheight}{0.15\columnwidth}
\newcommand{\subplotlinewidth}{1.5pt}

\definecolor{mycolor1}{rgb}{0.00000,0.44706,0.74118}%
\definecolor{mycolor2}{rgb}{0.85098,0.32549,0.09804}%
\begin{tabular}{c}
\begin{tikzpicture}
\pgfplotsset{%
    width=1.3\textwidth,
    label style = {font=\footnotesize},
    legend style = {font=\footnotesize},
    tick label style = {font=\scriptsize},
    title style =  {font=\scriptsize\bfseries},
    every axis label = {font=\footnotesize},
    legend image code/.code={
    \draw[mark repeat=2,mark phase=2]
    plot coordinates {
    (0cm,0cm)
    (0.15cm,0cm)        %
    (0.3cm,0cm)         %
    };%
    }
}
\begin{axis}[%
legend columns=2,
name=cat,
width=3.5cm,
height=2cm,
scale only axis,
ylabel={$\%$ Correct},
xlabel={Geodesic Error},
xmin=0,
xmax=1,
ylabel near ticks,
xtick={  0, 0.5,   1},
ymin=0,
ymax=103,
ytick={  0,  50, 100},
axis x line*=bottom,
axis y line*=left,
legend style={at={(1,1.1)},anchor=south east}
]
\addplot [color=mycolor1, smooth, line width=\subplotlinewidth]
  table[row sep=crcr]{%
0	3.34453045129254\\
0.01	5.82736965094202\\
0.02	13.9550167956769\\
0.03	22.6449539944501\\
0.04	31.8095516284504\\
0.05	39.3310939097415\\
0.06	46.1735066452461\\
0.07	52.2637651526216\\
0.08	56.6452460931795\\
0.09	59.9459617350665\\
0.1	62.9545786475829\\
0.11	65.4447203154666\\
0.12	67.4382941434205\\
0.13	69.5779173360596\\
0.14	71.0238060464437\\
0.15	72.1045713451146\\
0.16	73.0611946838031\\
0.17	73.8133489119322\\
0.18	74.3829414342048\\
0.19	74.9963487658829\\
0.2	75.4271943917044\\
0.21	75.7265955893092\\
0.22	75.9748795092741\\
0.23	76.193953556302\\
0.24	76.3984226668614\\
0.25	76.5955893091865\\
0.26	76.8803855703228\\
0.27	77.1067620855849\\
0.28	77.3112311961443\\
0.29	77.5741200525778\\
0.3	77.9100335913539\\
0.31	78.1948298524901\\
0.32	78.4212063677523\\
0.33	78.6840952241858\\
0.34	78.8593544618081\\
0.35	79.0565211041332\\
0.36	79.2244778735213\\
0.37	79.3924346429093\\
0.38	79.5165766028918\\
0.39	79.7356506499197\\
0.4	79.8743975463707\\
0.41	80.0934715933986\\
0.42	80.2468234263181\\
0.43	80.327150576895\\
0.44	80.3855703227691\\
0.45	80.4439900686432\\
0.46	80.5973419015627\\
0.47	80.7141813933109\\
0.48	80.8237184168249\\
0.49	80.9113480356361\\
0.5	81.0135825909157\\
0.51	81.0793048050241\\
0.52	81.1523294873667\\
0.53	81.2326566379436\\
0.54	81.298378852052\\
0.55	81.3494961296918\\
0.56	81.4006134073317\\
0.57	81.4225208120345\\
0.58	81.4517306849715\\
0.59	81.4736380896743\\
0.6	81.4809405579086\\
0.61	81.5101504308456\\
0.62	81.5247553673142\\
0.63	81.5393603037827\\
0.64	81.5539652402512\\
0.65	81.5685701767197\\
0.66	81.5831751131883\\
0.67	81.5977800496568\\
0.68	81.6342923908281\\
0.69	81.6342923908281\\
0.7	81.6561997955309\\
0.71	81.6708047319994\\
0.72	81.6708047319994\\
0.73	81.6927121367022\\
0.74	81.714619541405\\
0.75	81.7365269461078\\
0.76	81.7365269461078\\
0.77	81.7365269461078\\
0.78	81.7511318825763\\
0.79	81.7511318825763\\
0.8	81.7730392872791\\
0.81	81.7803417555134\\
0.82	81.7949466919819\\
0.83	81.7949466919819\\
0.84	81.8168540966847\\
0.85	81.8168540966847\\
0.86	81.8168540966847\\
0.87	81.8460639696217\\
0.88	81.8460639696217\\
0.89	81.8460639696217\\
0.9	81.8460639696217\\
0.91	81.853366437856\\
0.92	81.853366437856\\
0.93	81.8606689060903\\
0.94	81.8606689060903\\
0.95	81.8606689060903\\
0.96	81.8606689060903\\
0.97	81.8606689060903\\
0.98	81.8606689060903\\
1	81.8606689060903\\
};
\addlegendentry{{\textcolor{black}{Windheuser~\cite{windheuser2011}}}}

\addplot [color=mycolor2, smooth, line width=\subplotlinewidth]
  table[row sep=crcr]{%
0	3.90257717360521\\
0.01	7.14245256301331\\
0.02	17.1622769753611\\
0.03	28.7567261399037\\
0.04	40.1585952987822\\
0.05	50.1387708864344\\
0.06	58.8275276125743\\
0.07	65.7490795808553\\
0.08	71.3848768054376\\
0.09	75.5083545737751\\
0.1	78.8785046728972\\
0.11	81.619937694704\\
0.12	83.7496459926366\\
0.13	86.0152931180969\\
0.14	87.714528462192\\
0.15	88.8870008496177\\
0.16	90.0594732370433\\
0.17	90.9430756159728\\
0.18	91.5434721042198\\
0.19	91.9796091758709\\
0.2	92.3930897762673\\
0.21	92.8292268479184\\
0.22	93.1634097989238\\
0.23	93.3899745114698\\
0.24	93.6108751062022\\
0.25	93.8544321721892\\
0.26	94.035683942226\\
0.27	94.1886151231946\\
0.28	94.3188898329085\\
0.29	94.4661568960634\\
0.3	94.6247521948457\\
0.31	94.6983857264231\\
0.32	94.9022939677145\\
0.33	95.0325686774285\\
0.34	95.2251486830926\\
0.35	95.3724157462475\\
0.36	95.5480033984707\\
0.37	95.7519116397621\\
0.38	95.8425375247805\\
0.39	96.0521098838856\\
0.4	96.2276975361087\\
0.41	96.3409798923818\\
0.42	96.4485981308411\\
0.43	96.6241857830643\\
0.44	96.7148116680827\\
0.45	96.8450863777967\\
0.46	96.9753610875106\\
0.47	97.0433305012744\\
0.48	97.1509487397338\\
0.49	97.3378646275843\\
0.5	97.5077881619938\\
0.51	97.5984140470122\\
0.52	97.666383460776\\
0.53	97.7230246389125\\
0.54	97.8986122911357\\
0.55	97.9495893514585\\
0.56	98.023222883036\\
0.57	98.0968564146134\\
0.58	98.1421693571226\\
0.59	98.2044746530728\\
0.6	98.2611158312093\\
0.61	98.3347493627867\\
0.62	98.380062305296\\
0.63	98.4480317190598\\
0.64	98.6236193712829\\
0.65	98.6576040781648\\
0.66	98.7255734919286\\
0.67	98.7878787878788\\
0.68	98.8558482016426\\
0.69	98.8898329085245\\
0.7	98.9578023222883\\
0.71	99.0654205607477\\
0.72	99.1447182101388\\
0.73	99.2240158595299\\
0.74	99.2693288020391\\
0.75	99.3259699801756\\
0.76	99.3712829226848\\
0.77	99.4222599830076\\
0.78	99.5015576323988\\
0.79	99.5242141036533\\
0.8	99.6035117530445\\
0.81	99.660152931181\\
0.82	99.6998017558765\\
0.83	99.7394505805721\\
0.84	99.8017558765222\\
0.85	99.8187482299632\\
0.86	99.8244123477768\\
0.87	99.869725290286\\
0.88	99.8810535259133\\
0.89	99.9376947040498\\
0.9	99.9830076465591\\
0.91	99.9886717643727\\
0.92	99.9886717643727\\
0.93	99.9943358821863\\
0.94	99.9943358821863\\
0.95	99.9943358821863\\
0.96	100\\
0.97	100\\
0.98	100\\
1	100\\
};
\addlegendentry{{\textcolor{black}{Ours}}}
\end{axis}

\end{tikzpicture}%

\end{tabular}

%% file: figures/pckCombined.tex
\definecolor{mycolor1}{rgb}{0.00000,0.44706,0.74118}%
\definecolor{mycolor2}{rgb}{0.85098,0.32549,0.09804}%

\newcommand{\subplotwidth}{0.24\columnwidth}
\newcommand{\subplotheight}{0.16\columnwidth}
\newcommand{\subplotlinewidth}{1.5pt}

\definecolor{mycolor1}{rgb}{0.00000,0.44706,0.74118}%
\definecolor{mycolor2}{rgb}{0.92941,0.69412,0.12549}%
\definecolor{mycolor3}{rgb}{0.85098,0.32549,0.09804}%
\begin{center}
\ref{leg:pckCombinedLeg}
\end{center}
\definecolor{mycolor1}{rgb}{0.00000,0.44706,0.74118}%
\definecolor{mycolor2}{rgb}{0.92941,0.69412,0.12549}%
\definecolor{mycolor3}{rgb}{0.85098,0.32549,0.09804}%
\begin{tikzpicture}
\pgfplotsset{%
    width=1.3\textwidth,
    label style = {font=\tiny},
    legend style = {font=\scriptsize},
    tick label style = {font=\tiny},
    title style =  {font=\scriptsize\bfseries},
    every axis label = {font=\tiny},
}
\matrix{
    \input{figures/pckToscaAll}
	&\input{figures/pckKidsAll}
	&\input{figures/pckShrec19all} \\
};
\end{tikzpicture}%

%% file: appendix.tex
\section{Product Triangles to Explore}
In the following we explain the neighborhood between product triangles.
For the $f$-th product triangle $(a_1b_1, a_2b_2, a_3b_3)$ 
the set of neighbors $F_{\mathcal{N}(f)}$ is defined as
\begin{equation*}
\begin{aligned}
   &F_{\mathcal{N}(f)} =\\
   &\left\{
   \left(\begin{array}{c}
   a'_1 b'_1\\ a'_2 b'_2\\ a_3' b_3'
   \end{array}\right) 
   \in F :
   \begin{array}{c}
   (a_1' b_1' = a_2 b_2 \wedge a'_2 b'_2 = a_1 b_1) \vee \\
   (a_1' b_1' = a_3 b_3 \wedge a'_2 b'_2 = a_2 b_2) \vee \\
   (a_1' b_1' = a_1 b_1 \wedge a'_2 b'_2 = a_3 b_3) \vee \\
   (a_2' b_2' = a_2 b_2 \wedge a'_3 b'_3 = a_1 b_1) \vee \\
   (a_2' b_2' = a_3 b_3 \wedge a'_3 b'_3 = a_2 b_2) \vee \\
   (a_2' b_2' = a_1 b_1 \wedge a'_3 b'_3 = a_3 b_3) \vee \\
   (a_3' b_3' = a_2 b_2 \wedge a'_1 b'_1 = a_1 b_1) \vee \\
   (a_3' b_3' = a_3 b_3 \wedge a'_1 b'_1 = a_2 b_2) \vee \\
   (a_3' b_3' = a_1 b_1 \wedge a'_1 b'_1 = a_3 b_3) \phantom{\vee}
   \end{array}
   \right\}.
   \end{aligned}
\end{equation*}
In words, every product triangle which shares an opposite oriented edge with the $f$-th product triangle is neighboring to the $f$-th product triangle.
The union of all sets $F_{\mathcal{N}(f)}$ yields the set of exploration candidates 
\begin{equation}
    F_\text{expl} = \bigcup_{f \text{ part of solution}} F_{\mathcal{N}(f)}.
\end{equation}
When searching for new matchings with our primal heuristic, we only iterate over the product triangles in $F_\text{expl}$.
If none of the product triangles in $F_\text{expl}$ is feasible, the current partial solution cannot be rounded to a feasible solution and previously added matchings are removed.

\section{Recomputation of Min-Marginals}
In order to obtain a rounded primal solution we repeatedly recompute the min-marginals after a certain number of calls of the primal heuristic.
To this end, we make a trade-off between computation time and quality of min-marginals and achievable solutions.
For $k$ being the total number of calls of the primal heuristic, we compute the threshold $\alpha = 0.2 \cdot \min(\abs{F_X}, \abs{F_Y})$, and whenever we have added $k{\cdot}\alpha$ product triangles to the solution, we re-compute the min-marginals.
For that, we first fix variables in~\eqref{eq:opt-prob-discrete}, then dualize \eqref{eq:opt-prob-discrete}, and eventually solve the dual problem again.

\section{Comparison to Windheuser et al.}
\paragraph{Reimplementation of Windheuser et al.'s approach.}
Windheuser et al.~tackle the ILP formulation~\eqref{eq:opt-prob-discrete}  through an LP-relaxation, for which variables are gradually rounded to binary values and then kept fixed.
This process is repeated until all constraints are fulfilled. The LP-relaxation reads
\begin{equation} 
    \underset{\Gamma \in [0, 1]^{|F|}}{\text{min}} \mathbbm{E}^\top  \Gamma  
    ~~  \text{s.t.} ~~
    \begin{pmatrix}
      \pi_X \\ \pi_Y \\ \partial
    \end{pmatrix}
    \Gamma
    = 
    \begin{pmatrix}
       \boldsymbol{1}_{|F_X|} \\ \boldsymbol{1}_{|F_Y|} \\ \boldsymbol{0}_{\abs{E}} \\
    \end{pmatrix}.
\label{eq:lp-relaxation}
\end{equation}
In Algorithm~\ref{algo:windheuser-reimplementation} we sketch our re-implementation of the approach by Windheuser et al., which follows the procedure explained in \cite{windheuser2011}. For solving the LP-relaxation we use the state-of-the-art LP-solver Gurobi \cite{gurobi}.

\begin{algorithm}[h]
	\caption{Solving \eqref{eq:opt-prob-discrete} according to Windheuser et al.}\label{algo:windheuser-reimplementation}
	\SetAlgoLined
	\KwIn{\eqref{eq:lp-relaxation}}
	\KwOut{Solution $\Gamma \in \{0, 1\}^{|F|}$}
	\While{Constraints not fulfilled}{
	    Solve LP~\ref{eq:lp-relaxation} (while keeping already set elements of $\Gamma$ fixed)\;
	    \If{$\Gamma_i > 0.5$}{
	        Fix $\Gamma_i = 1$\;
	    }
	}
\end{algorithm}

\paragraph{Test setup.}
In Fig.~\ref{fig:toscaoursvswindheuser} of the main paper we quantitatively compare our solver to the  approach by Windheuser et al. We complement these results with Fig.~\ref{fig:toscaoursvswindheuser-full}, where we show the corresponding curves for the individual classes. In this experiment for each shape matching instance we first run our method, and afterwards run the approach  by Windheuser et al.~with a fixed time budget. This is implemented by tracking the total time of Algorithm~\ref{algo:windheuser-reimplementation} (including within the LP-solver itself),
and once the limit is reached the algorithm is terminated and the current (possibly partial) solution is used as matching.

In all experiments we allow the method of Windheuser et al.~to use $10 \times$ more time than ours -- even with this generous time budget the method by Windheuser et al.~is often not able to produce a complete matching.

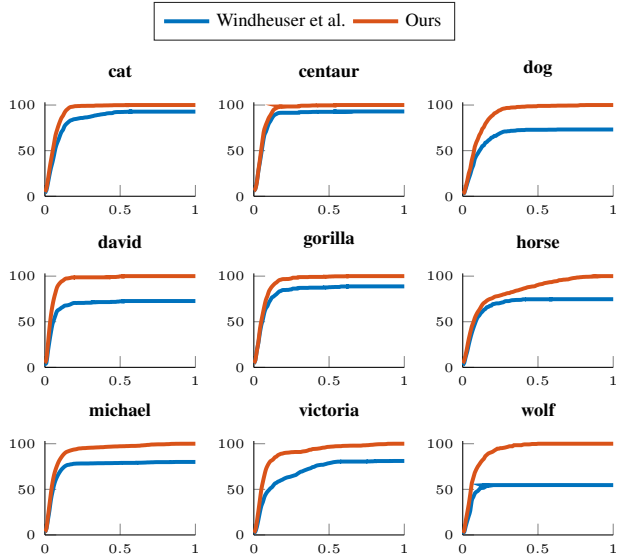
\begin{figure}[h!]
	\begin{center}
		\input{figures/pckWindVsOursTimelimit} 
		\caption{Comparison of the percentage of correct matchings for different shape classes of the TOSCA dataset. The horizontal axis shows the geodesic error threshold, and the vertical axis shows the percentage of matches that are smaller than or equal to this error. We reduce all shapes to $175$ triangles. For Windheuser et al.~we allow the solver to take $10\times$ more time than our method needed. The curves by Windheuser et al. are lower because often it only finds only few matchings within the given time budget.}
		\label{fig:toscaoursvswindheuser-full}. 
	\end{center}
\end{figure}

\section{Partial Shape Matching}
While the original formalism~\eqref{eq:opt-prob-discrete} assumes that the shapes do not have a boundary, we propose a simple yet effective way of dealing with partial shapes.
To this end, we simply close existing holes with triangular patches and then compute our energy for each product triangle that involves a `hole-triangle'.
In Fig.~\ref{fig:partial-wolf-holes} we show an example of shapes with closed holes and their corresponding shapes with holes.
We note that in all other visualizations we do not show the closed holes but highlight all shape boundaries.
\begin{figure}[H]
\begin{center}
    \includegraphics[width=0.6\columnwidth]{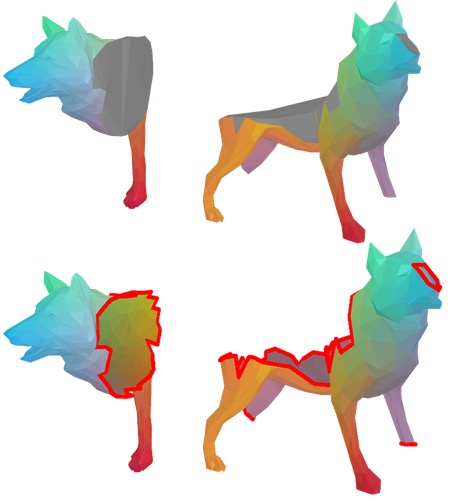}
    \caption{Matching of partial wolf shapes with our approach. At the top, we additionally show the closed shapes (hole-closing triangles are plotted in gray).}
    \label{fig:partial-wolf-holes}
\end{center}
\end{figure}

\section{Additional Results}
\paragraph{Additional qualitative results.}
In Fig.~\ref{fig:against-sota-qualitative-tosca} we show additional qualitative matchings of the method by Eisenberger et al.~\cite{eisenberger2020}, Ren et al.~\cite{ren2021} and ours. The experimental setting corresponds to  Fig.~\ref{fig:against-sota-qualitative} of the main paper. 
\begin{figure*}[ht!]
	\begin{center}
		\includegraphics[width=\linewidth]{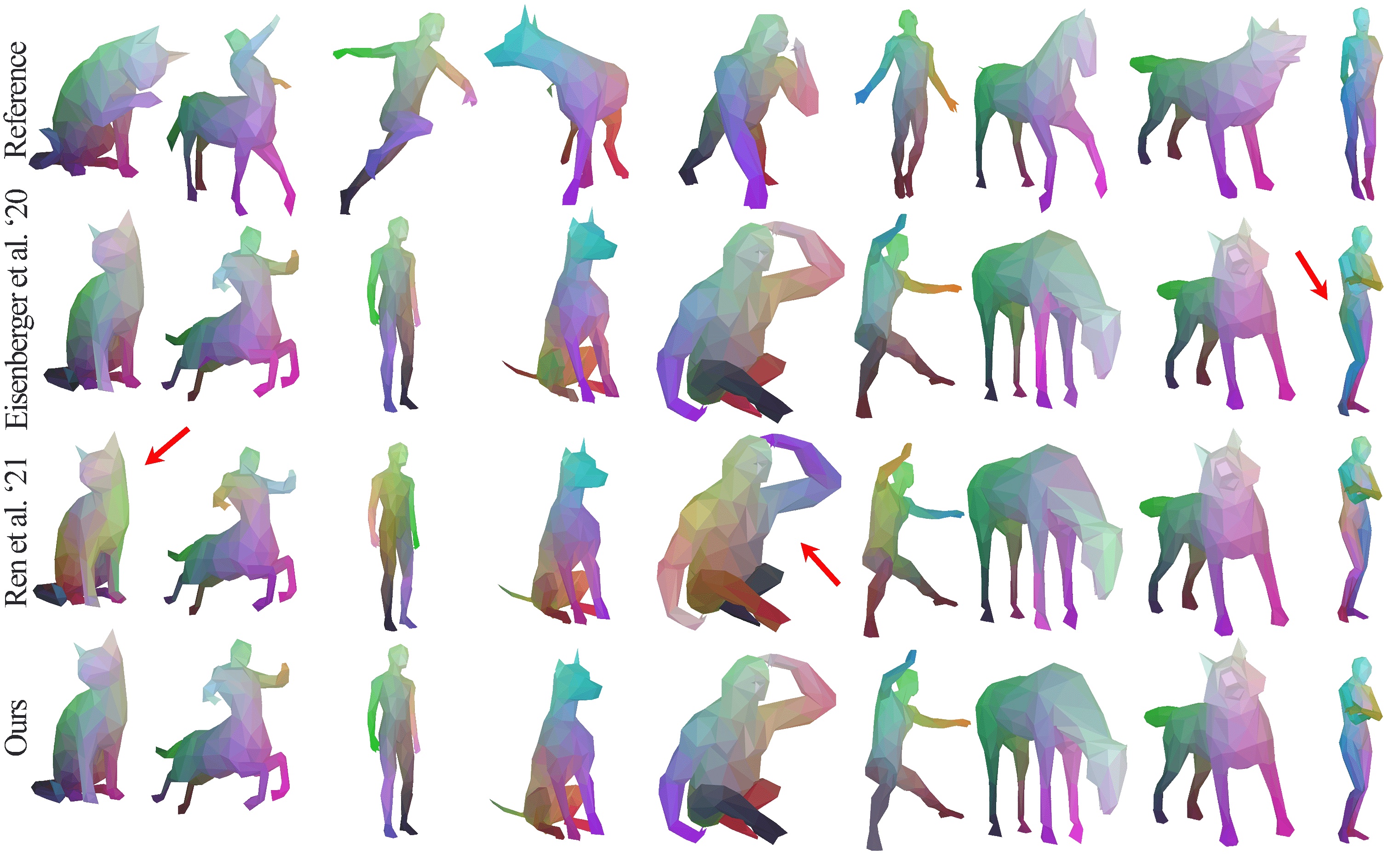}
		\caption{\textbf{Qualitative comparison} of the method by Eisenberger et al.~\cite{eisenberger2020} (second row), Ren et al.~\cite{ren2021} (third row) and Ours (last row) on the TOSCA dataset.
		}
		\label{fig:against-sota-qualitative-tosca}
	\end{center}
\end{figure*}

\paragraph{Additional quantitative results.}
In Fig.~\ref{fig:matching-accuracy-detailed} we show error curves for individual shape classes of TOSCA and KIDS dataset.

\paragraph{Shape resolution and discretization.} 
A strong advantage of the utilized discrete matching model is that by allowing for degenerate matchings (triangle-vertex and triangle-edge matchings) it can handle different shape resolutions and discretizations. In Fig.~\ref{fig:dummy-discret} we show this for shapes with different discretisation (left), as well as for substantially varying resolution (right, factor of ${\approx} 3\times$ more triangles).
    \newcommand{\dummyheight}[0]{1.3cm}
\begin{figure}[H]
	\begin{center}
	\footnotesize
	\begin{tabular}{cc}
	\setlength\tabcolsep{0pt}
	\includegraphics[height=\dummyheight]{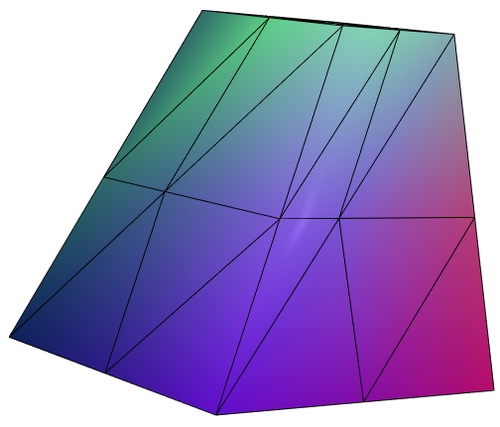}\includegraphics[height=\dummyheight]{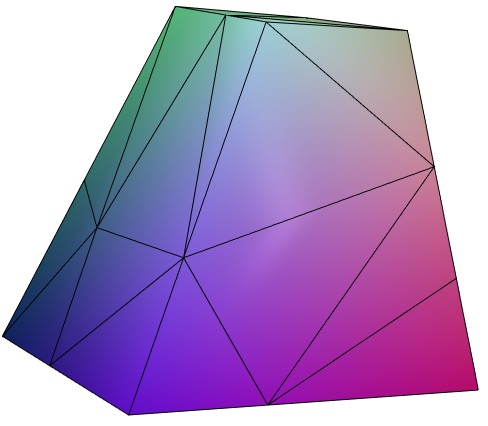} ~~~~&~~~~
    \includegraphics[height=\dummyheight]{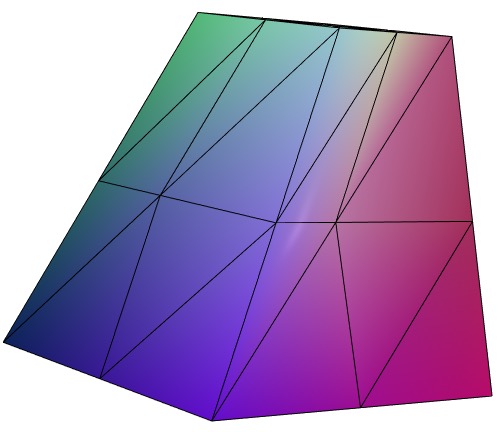}\includegraphics[height=\dummyheight]{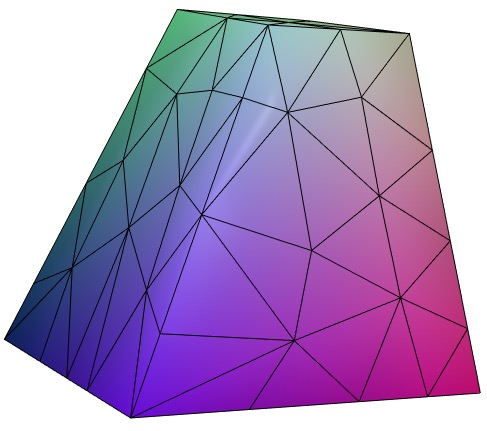}
	\end{tabular}
	\end{center}
	\caption{Matching shapes with different discretisation for two shape pairs.}
    \label{fig:dummy-discret}
\end{figure}
In Fig.~\ref{fig:partial-discret} we show additional partial-to-partial results with shapes of different mesh resolution (factor of ${\approx}\sqrt{2}$) for different levels of partiality.
\newcommand{\dummypartialheight}[0]{1.1cm}
\begin{figure}[H]
	\begin{center}
	\footnotesize
	\begin{tabular}{ccccc}
	\setlength\tabcolsep{0pt}
	Reference & 80\% & 70\% & 60\% & 50\%\\
    \includegraphics[height=\dummypartialheight]{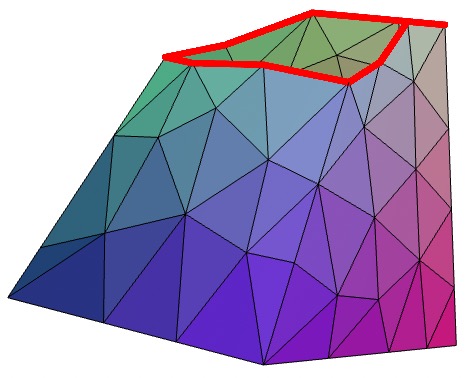} &
    \includegraphics[height=\dummypartialheight]{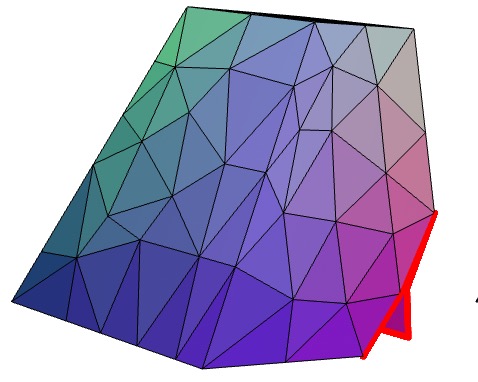} &
    \includegraphics[height=\dummypartialheight]{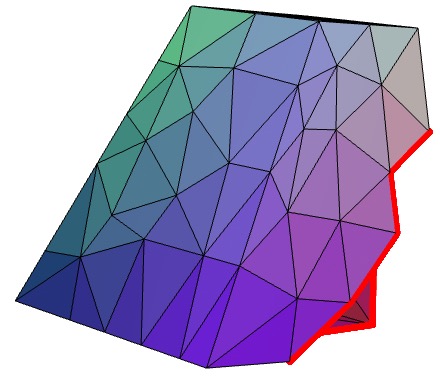} &
    \includegraphics[height=\dummypartialheight]{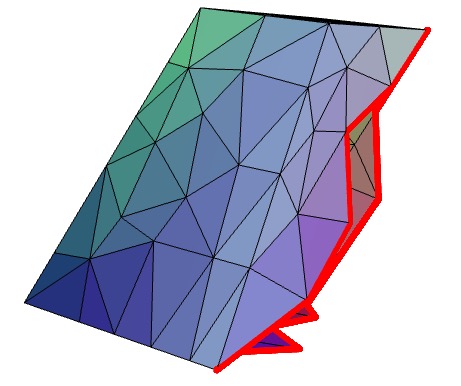} &
    \includegraphics[height=\dummypartialheight]{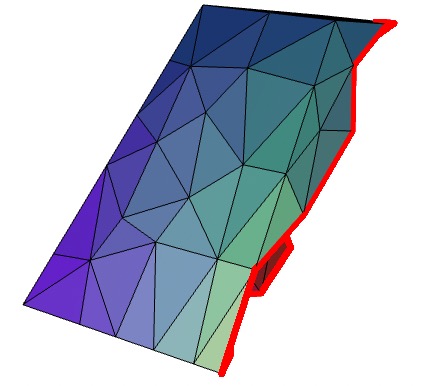} 
	\end{tabular}
	\end{center}
	\caption{Matching shapes with different levels of partiality with different mesh resolution. At the partiality level of around 50~\% of the original shape no plausible matching can be found anymore.}
	\label{fig:partial-discret}
\end{figure}

\paragraph{Non-isometries.}
The discrete diffeomorphism implemented by the \emph{constraints} in (ILP-SM) can also handle non-isometries, which we show in Fig.~\ref{fig:shrec20} on the SHREC'20 dataset~\cite{dyke2020track}.
\begin{figure}[H]
	\begin{center}
	\footnotesize
	\begin{tabular}{ccc}
	\setlength\tabcolsep{0pt}
	\includegraphics[height=2.2cm]{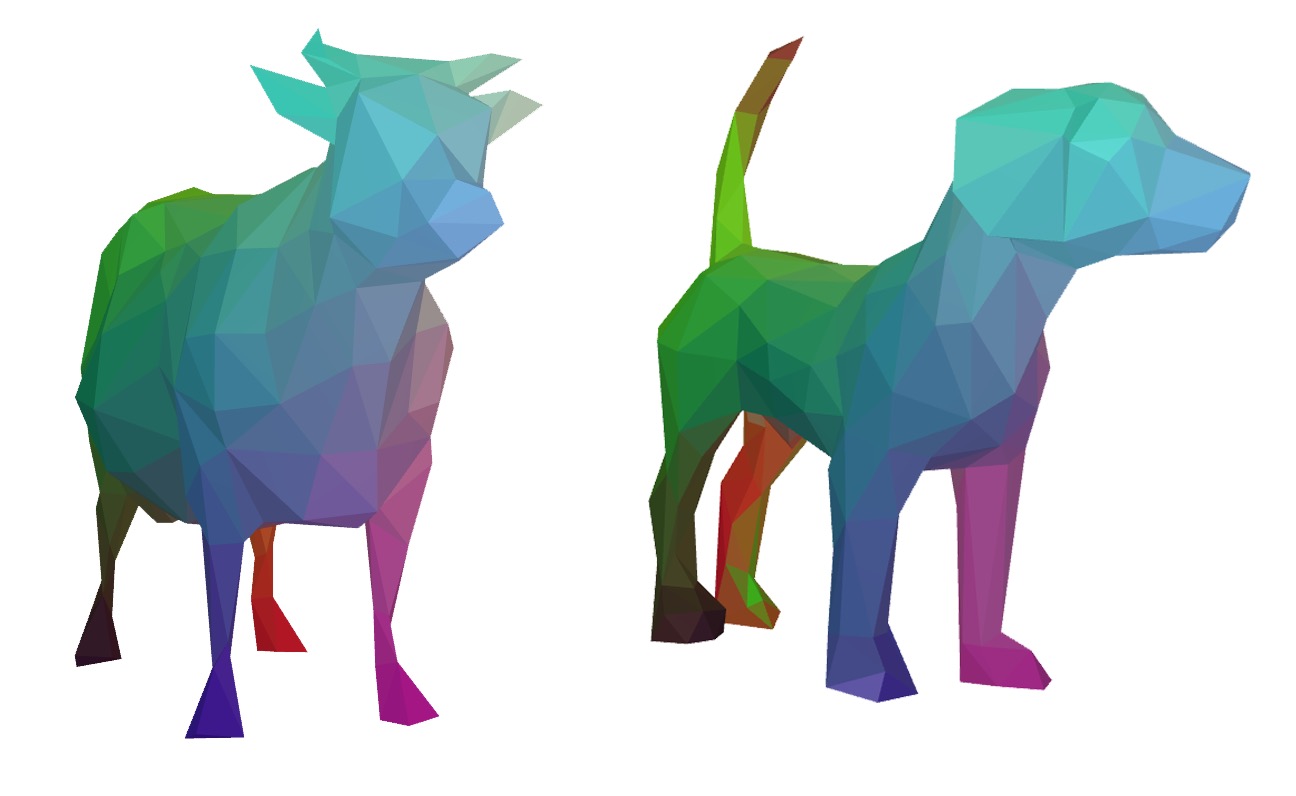} ~~&~~
    \includegraphics[height=2.2cm]{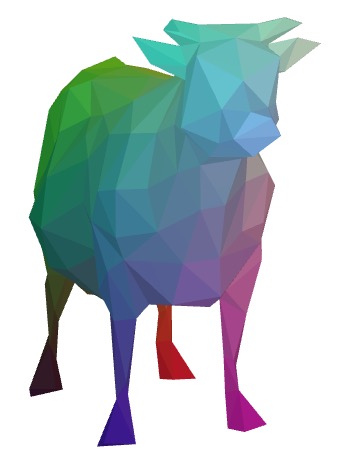}&\includegraphics[height=2.4cm]{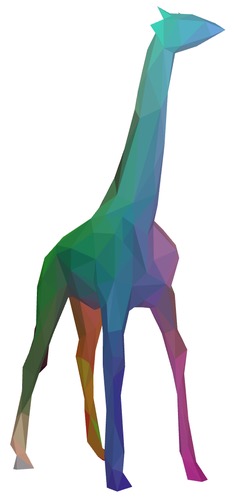}
	\end{tabular}
 	\vspace{-4mm}
	\end{center}
	\caption{Matching two different non-isometric shape pairs of the SHREC'20 dataset.}
	\label{fig:shrec20}
\end{figure}

\begin{figure}[h!]
	\begin{center}
		\input{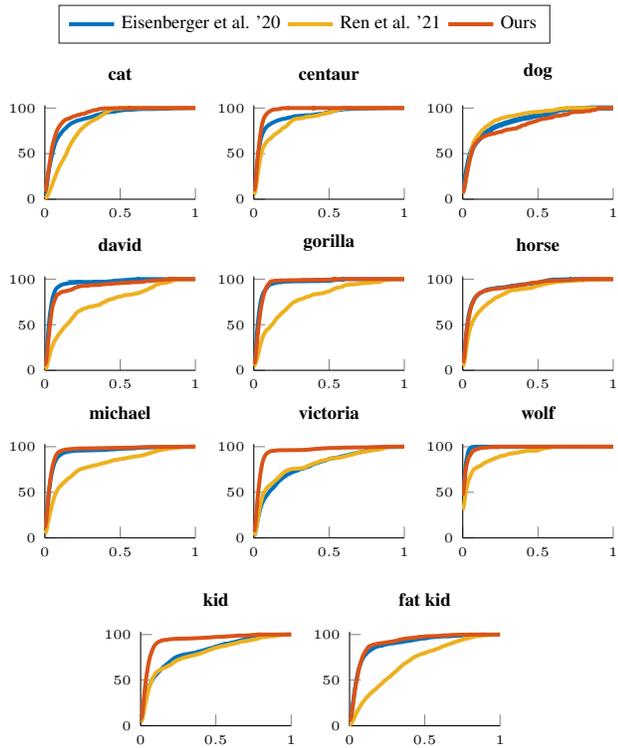}
	    \input{figures/pckKids}
		\caption{PCK curves for individual shape classes of \textbf{TOSCA dataset} (first three rows) and \textbf{KIDS dataset} (last row).
		The horizontal axis shows the geodesic error threshold, and the vertical axis shows the percentage of matches smaller than or equal to this error.}
		\label{fig:matching-accuracy-detailed}
	\end{center}
\end{figure}

\paragraph{Texture transfer.}
In Fig.~\ref{fig:texturetransfer} we illustrate that the matchings computed with our method can be used for texture transfer.
\begin{figure}[h!]
	\begin{center}
	\footnotesize
	\begin{tabular}{cc}
	\setlength\tabcolsep{0pt}
	\includegraphics[height=1.9cm]{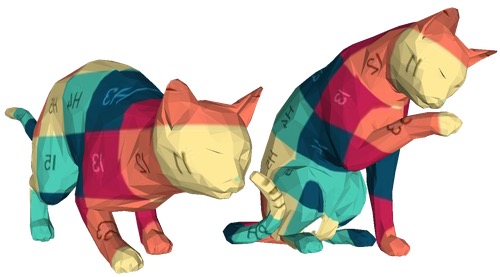} ~~&~~
    \includegraphics[height=2.4cm]{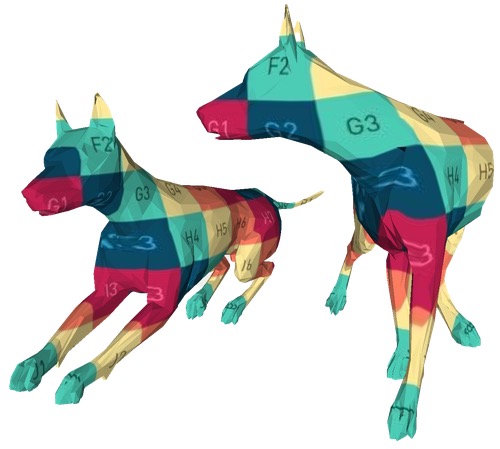}
	\end{tabular}
	\end{center}
	\caption{Texture transfer based on correspondences computed with our method.}
	\label{fig:texturetransfer}
\end{figure}

\paragraph{Error cases.}
In Fig.~\ref{fig:fails} we show some failure modes of our method. For the partial-partial dog shown in Fig.~\ref{fig:partial-dog-fail} ours was correctly initialized (correct matching head), but was not able to determine the remaining part of matching appropriately. This could possibly be accounted for by considering a tighter relaxation. For the partial-partial matching of the  `Victoria' shape in Fig.~\ref{fig:partial-vici-fail}, the overlapping areas of both shapes are too small, so that finding a proper matching is extremely challenging. The matching of the dog in Fig.~\ref{fig:dog-fail} failed due to a wrong initialization. This may happen if the total min-marginals do not clearly indicate which matching of initial matchings are best suited.
As illustrated by these examples, there are some cases in which our method may fail. Nevertheless, as the quantitative experiments in the main paper indicate, overall our obtained matchings improve upon several existing state-of-the-art methods for non-rigid shape matching.
\newcommand{\failureheight}[0]{5cm}
\begin{figure}[h!]
    \begin{subfigure}[b]{0.3\columnwidth}
    \centering
      \includegraphics[height=\failureheight]{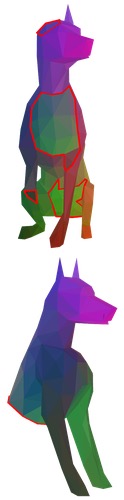}
      \caption{}
      \label{fig:partial-dog-fail}
    \end{subfigure}
    \hfill
    \begin{subfigure}[b]{0.3\columnwidth}
    \centering
      \includegraphics[height=\failureheight]{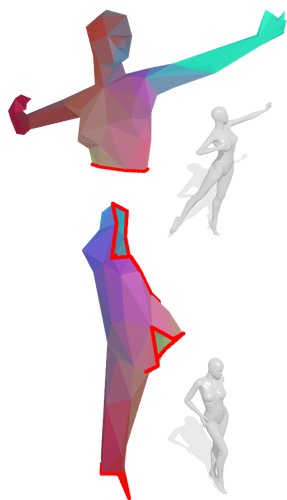}
      \caption{}
      \label{fig:partial-vici-fail}
    \end{subfigure}
    \hfill
    \begin{subfigure}[b]{0.3\columnwidth}
    \centering
      \includegraphics[height=\failureheight]{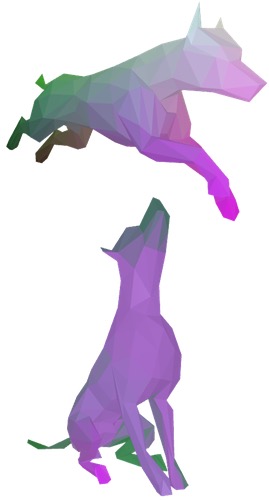}
      \caption{}
      \label{fig:dog-fail}
    \end{subfigure}
    \caption{Some failure modes of our approach.
    }
    \label{fig:fails}
\end{figure}

%% file: figures/pckWindVsOursTimelimit.tex
\definecolor{mycolor1}{rgb}{0.00000,0.44706,0.74118}%
\definecolor{mycolor2}{rgb}{0.85098,0.32549,0.09804}%

\newcommand{\subplotwidth}{0.24\columnwidth}
\newcommand{\subplotheight}{0.15\columnwidth}
\newcommand{\subplotlinewidth}{1.5pt}

\definecolor{mycolor1}{rgb}{0.00000,0.44706,0.74118}%
\definecolor{mycolor2}{rgb}{0.85098,0.32549,0.09804}%
\begin{center}
\ref{plotlegend}
\end{center}
\begin{tikzpicture}
\pgfplotsset{%
    width=1.3\textwidth,
    label style = {font=\tiny},
    legend style = {font=\scriptsize},
    tick label style = {font=\tiny},
    title style =  {font=\scriptsize\bfseries},
    every axis label = {font=\tiny},
}
\matrix{
\begin{axis}[%
legend columns=2,
name=cat,
width=\subplotwidth,
height=\subplotheight,
scale only axis,
xmin=0,
xmax=1,
xtick={  0, 0.5,   1},
ymin=0,
ymax=103,
ytick={  0,  50, 100},
title={cat},
axis x line*=bottom,
axis y line*=left,
legend style={legend cell align=left, align=left,
legend to name=plotlegend,
draw=white!15!black}
]
\addplot [color=mycolor1, smooth, line width=\subplotlinewidth]
  table[row sep=crcr]{%
0	3.69761709120789\\
0.0100000000000051	6.61462612982744\\
0.019999999999996	14.5028759244043\\
0.0300000000000011	23.1306491372227\\
0.0400000000000063	30.936729663106\\
0.0499999999999972	37.1815940838127\\
0.0600000000000023	44.6589975349219\\
0.0699999999999932	52.670501232539\\
0.0799999999999983	57.9293344289236\\
0.0900000000000034	61.873459326212\\
0.0999999999999943	66.5160230073952\\
0.109999999999999	70.5423171733772\\
0.120000000000005	72.7608874281019\\
0.129999999999995	76.7460969597371\\
0.140000000000001	78.8824979457683\\
0.150000000000006	80.4437140509449\\
0.159999999999997	81.8816762530814\\
0.170000000000002	83.0731306491372\\
0.180000000000007	83.5661462612983\\
0.189999999999998	84.3056696795398\\
0.200000000000003	84.5932621199671\\
0.209999999999994	84.8397699260477\\
0.219999999999999	85.2095316351684\\
0.230000000000004	85.5382087099425\\
0.239999999999995	85.8258011503698\\
0.25	85.8258011503698\\
0.260000000000005	86.1955628594906\\
0.269999999999996	86.4420706655711\\
0.280000000000001	86.5242399342646\\
0.290000000000006	86.8529170090386\\
0.299999999999997	87.1405094494659\\
0.310000000000002	87.5924404272802\\
0.319999999999993	87.9211175020542\\
0.329999999999998	88.455217748562\\
0.340000000000003	88.7017255546426\\
0.349999999999994	89.0714872637633\\
0.359999999999999	89.2769104354971\\
0.370000000000005	89.6466721446179\\
0.379999999999995	89.8520953163517\\
0.390000000000001	90.0986031224322\\
0.400000000000006	90.2218570254725\\
0.409999999999997	90.7559572719803\\
0.420000000000002	90.961380443714\\
0.430000000000007	91.0846343467543\\
0.439999999999998	91.0846343467543\\
0.450000000000003	91.2489728841413\\
0.459999999999994	91.9884963023829\\
0.480000000000004	92.3171733771569\\
0.5	92.3171733771569\\
0.510000000000005	92.3993426458505\\
0.519999999999996	92.3993426458505\\
0.530000000000001	92.5225965488907\\
0.540000000000006	92.5225965488907\\
0.549999999999997	92.645850451931\\
0.560000000000002	92.645850451931\\
0.569999999999993	92.7280197206245\\
0.579999999999998	92.7280197206245\\
0.590000000000003	92.7691043549712\\
1	92.7691043549712\\
};
\addlegendentry{{\textcolor{black}{Windheuser et al.}}}

\addplot [color=mycolor2, smooth, line width=\subplotlinewidth]
  table[row sep=crcr]{%
0	4.64054773678204\\
0.0100000000000051	8.74857360213009\\
0.019999999999996	18.1057436287562\\
0.0300000000000011	29.1745910992773\\
0.0400000000000063	39.1403575503994\\
0.0600000000000023	57.1700266260936\\
0.0699999999999932	66.5652339292507\\
0.0799999999999983	72.3088626854317\\
0.0900000000000034	77.5199695701788\\
0.0999999999999943	81.9703309243058\\
0.109999999999999	85.8881704069989\\
0.120000000000005	88.2845188284519\\
0.129999999999995	91.5557246101179\\
0.140000000000001	93.9901103081019\\
0.150000000000006	95.473564092811\\
0.159999999999997	96.918980600989\\
0.170000000000002	97.793837961202\\
0.180000000000007	98.060098896919\\
0.189999999999998	98.364397109167\\
0.200000000000003	98.744769874477\\
0.209999999999994	98.782807151008\\
0.219999999999999	98.896918980601\\
0.230000000000004	98.896918980601\\
0.239999999999995	99.049068086725\\
0.25	99.087105363256\\
0.269999999999996	99.087105363256\\
0.280000000000001	99.163179916318\\
0.319999999999993	99.163179916318\\
0.329999999999998	99.277291745911\\
0.340000000000003	99.467478128566\\
0.379999999999995	99.467478128566\\
0.390000000000001	99.581589958159\\
0.400000000000006	99.581589958159\\
0.409999999999997	99.61962723469\\
0.420000000000002	99.61962723469\\
0.430000000000007	99.695701787752\\
0.439999999999998	99.695701787752\\
0.450000000000003	99.733739064283\\
0.459999999999994	99.847850893876\\
0.519999999999996	99.847850893876\\
0.530000000000001	99.923925446938\\
0.560000000000002	99.923925446938\\
0.569999999999993	100\\
1	100\\
};
\addlegendentry{{\textcolor{black}{Ours}}}
\end{axis}

&
\begin{axis}[%
name=centaur,
width=\subplotwidth,
height=\subplotheight,
scale only axis,
xmin=0,
xmax=1,
xtick={  0, 0.5,   1},
ymin=0,
ymax=103,
ytick={  0,  50, 100},
title={centaur},
axis x line*=bottom,
axis y line*=left
]
\addplot [color=mycolor1, smooth, line width=\subplotlinewidth, forget plot]
  table[row sep=crcr]{%
0	6.08486789431545\\
0.0100000000000051	8.08646917534027\\
0.019999999999996	15.8526821457166\\
0.0300000000000011	27.3819055244195\\
0.0400000000000063	36.9895916733387\\
0.0499999999999972	47.3178542834268\\
0.0600000000000023	58.0464371497198\\
0.0699999999999932	67.093674939952\\
0.0799999999999983	70.6164931945556\\
0.0900000000000034	75.4203362690152\\
0.0999999999999943	79.8238590872698\\
0.109999999999999	82.3058446757406\\
0.120000000000005	85.3482786228983\\
0.129999999999995	87.8302642113691\\
0.140000000000001	89.351481184948\\
0.150000000000006	90.1521216973579\\
0.159999999999997	91.0328262610088\\
0.170000000000002	91.3530824659728\\
0.189999999999998	91.5132105684548\\
0.290000000000006	91.5132105684548\\
0.299999999999997	91.9135308246597\\
0.310000000000002	92.2337870296237\\
0.359999999999999	92.2337870296237\\
0.370000000000005	92.3138510808647\\
0.379999999999995	92.3138510808647\\
0.390000000000001	92.4739791833467\\
0.420000000000002	92.4739791833467\\
0.430000000000007	92.5540432345877\\
0.530000000000001	92.5540432345877\\
0.540000000000006	92.6341072858287\\
0.549999999999997	92.6341072858287\\
0.560000000000002	92.9543634907926\\
1	92.9543634907926\\
};
\addplot [color=mycolor2, smooth, line width=\subplotlinewidth, forget plot]
  table[row sep=crcr]{%
0	6.13409415121255\\
0.0100000000000051	8.13124108416548\\
0.019999999999996	15.2639087018545\\
0.0300000000000011	28.1740370898716\\
0.0400000000000063	39.586305278174\\
0.0600000000000023	64.0513552068474\\
0.0699999999999932	73.6091298145506\\
0.0900000000000034	82.2396576319544\\
0.0999999999999943	86.1626248216833\\
0.109999999999999	88.7303851640514\\
0.120000000000005	91.6547788873039\\
0.129999999999995	94.4365192582026\\
0.140000000000001	95.7917261055635\\
0.150000000000006	96.718972895863\\
0.159999999999997	97.5035663338088\\
0.170000000000002	98.1455064194008\\
0.180000000000007	98.3594864479315\\
0.189999999999998	98.4308131241084\\
0.290000000000006	98.4308131241084\\
0.299999999999997	98.8587731811698\\
0.310000000000002	99.1440798858773\\
0.319999999999993	99.2867332382311\\
0.359999999999999	99.2867332382311\\
0.379999999999995	99.4293865905849\\
0.390000000000001	99.5720399429387\\
0.420000000000002	99.5720399429387\\
0.430000000000007	99.6433666191156\\
0.530000000000001	99.6433666191156\\
0.540000000000006	99.7146932952924\\
0.549999999999997	99.7146932952924\\
0.560000000000002	100\\
1	100\\
};
\end{axis}

&
\begin{axis}[%
name=dog,
width=\subplotwidth,
height=\subplotheight,
scale only axis,
xmin=0,
xmax=1,
xtick={  0, 0.5,   1},
ymin=0,
ymax=103,
ytick={  0,  50, 100},
title={dog},
axis x line*=bottom,
axis y line*=left
]

\addplot [color=mycolor1, line width=\subplotlinewidth, forget plot]
  table[row sep=crcr]{%
0	2.45649948822927\\
0.0100000000000051	3.58239508700102\\
0.019999999999996	8.59774820880246\\
0.0300000000000011	14.2272262026612\\
0.0400000000000063	19.3449334698055\\
0.0499999999999972	24.0532241555783\\
0.0600000000000023	29.8874104401228\\
0.0699999999999932	33.9815762538383\\
0.0799999999999983	40.9416581371546\\
0.0900000000000034	44.4216990788127\\
0.0999999999999943	46.878198567042\\
0.109999999999999	49.6417604912999\\
0.120000000000005	52.0982599795292\\
0.129999999999995	54.96417604913\\
0.140000000000001	55.7830092118731\\
0.150000000000006	58.6489252814739\\
0.159999999999997	59.467758444217\\
0.170000000000002	61.207778915046\\
0.180000000000007	62.3336745138178\\
0.189999999999998	64.1760491299898\\
0.200000000000003	64.7901740020471\\
0.209999999999994	66.1207778915046\\
0.219999999999999	66.2231320368475\\
0.230000000000004	67.4513817809621\\
0.25	69.2937563971341\\
0.269999999999996	70.7267144319345\\
0.280000000000001	71.0337768679632\\
0.290000000000006	71.0337768679632\\
0.299999999999997	71.3408393039918\\
0.310000000000002	71.3408393039918\\
0.319999999999993	71.7502558853634\\
0.329999999999998	71.7502558853634\\
0.349999999999994	71.9549641760491\\
0.359999999999999	72.3643807574207\\
0.370000000000005	72.4667349027635\\
0.379999999999995	72.4667349027635\\
0.390000000000001	72.5690890481064\\
0.409999999999997	72.5690890481064\\
0.420000000000002	72.8761514841351\\
0.560000000000002	72.8761514841351\\
0.569999999999993	72.978505629478\\
0.620000000000005	72.978505629478\\
0.629999999999995	73.1832139201638\\
1	73.1832139201638\\
};
\addplot [color=mycolor2, line width=\subplotlinewidth, forget plot]
  table[row sep=crcr]{%
0	2.85035629453682\\
0.0100000000000051	3.95882818685669\\
0.019999999999996	10.3721298495645\\
0.0300000000000011	16.9437846397466\\
0.0400000000000063	25.0197941409343\\
0.0499999999999972	30.166270783848\\
0.0600000000000023	39.3507521773555\\
0.0699999999999932	45.7640538400633\\
0.0799999999999983	54.2359461599367\\
0.0900000000000034	58.9865399841647\\
0.0999999999999943	62.5494853523357\\
0.109999999999999	66.9041963578781\\
0.120000000000005	69.6753760886778\\
0.129999999999995	73.8717339667459\\
0.140000000000001	75.771971496437\\
0.150000000000006	79.3349168646081\\
0.159999999999997	81.7102137767221\\
0.170000000000002	84.1646872525732\\
0.180000000000007	85.9857482185273\\
0.189999999999998	87.6484560570071\\
0.200000000000003	88.9152810768013\\
0.209999999999994	90.8155186064925\\
0.219999999999999	91.6072842438638\\
0.230000000000004	92.6365795724465\\
0.239999999999995	93.9034045922407\\
0.25	94.4576405384006\\
0.260000000000005	95.4869358669834\\
0.269999999999996	96.0411718131433\\
0.280000000000001	96.2787015043547\\
0.290000000000006	96.7537608867775\\
0.299999999999997	96.8329374505146\\
0.310000000000002	96.8329374505146\\
0.319999999999993	97.2288202692003\\
0.340000000000003	97.3871733966746\\
0.349999999999994	97.3871733966746\\
0.359999999999999	97.8622327790974\\
0.370000000000005	97.9414093428345\\
0.379999999999995	97.9414093428345\\
0.390000000000001	98.0997624703088\\
0.400000000000006	98.1789390340459\\
0.409999999999997	98.1789390340459\\
0.420000000000002	98.4164687252573\\
0.430000000000007	98.5748218527316\\
0.459999999999994	98.5748218527316\\
0.469999999999999	98.6539984164687\\
0.489999999999995	98.6539984164687\\
0.5	99.0498812351544\\
0.579999999999998	99.0498812351544\\
0.590000000000003	99.2082343626287\\
0.620000000000005	99.2082343626287\\
0.629999999999995	99.3665874901029\\
0.689999999999998	99.3665874901029\\
0.700000000000003	99.4457640538401\\
0.760000000000005	99.4457640538401\\
0.769999999999996	99.6041171813143\\
0.780000000000001	99.9208234362629\\
0.810000000000002	99.9208234362629\\
0.819999999999993	100\\
1	100\\
};
\end{axis}
\\
\begin{axis}[%
name=david,
width=\subplotwidth,
height=\subplotheight,
scale only axis,
xmin=0,
xmax=1,
xtick={  0, 0.5,   1},
ymin=0,
ymax=103,
ytick={  0,  50, 100},
title={david},
axis x line*=bottom,
axis y line*=left
]
\addplot [color=mycolor1,smooth, line width=\subplotlinewidth, forget plot]
  table[row sep=crcr]{%
0	3.82626680455016\\
0.0100000000000051	5.27404343329886\\
0.019999999999996	17.5801447776629\\
0.0300000000000011	28.6452947259566\\
0.0400000000000063	38.5729058945191\\
0.0499999999999972	47.2595656670114\\
0.0600000000000023	52.3267838676319\\
0.0699999999999932	57.0837642192348\\
0.0799999999999983	61.5305067218201\\
0.0900000000000034	62.9782833505688\\
0.0999999999999943	64.0124095139607\\
0.109999999999999	65.3567735263702\\
0.120000000000005	66.287487073423\\
0.129999999999995	67.5284384694933\\
0.140000000000001	67.8386763185109\\
0.150000000000006	68.0455015511892\\
0.159999999999997	68.459152016546\\
0.170000000000002	69.0796277145812\\
0.189999999999998	70.5274043433299\\
0.209999999999994	70.5274043433299\\
0.219999999999999	70.6308169596691\\
0.260000000000005	70.6308169596691\\
0.269999999999996	70.8376421923475\\
0.299999999999997	70.8376421923475\\
0.310000000000002	70.9410548086867\\
0.319999999999993	71.2512926577042\\
0.329999999999998	71.3547052740434\\
0.340000000000003	71.5615305067218\\
0.430000000000007	71.5615305067218\\
0.439999999999998	71.7683557394002\\
0.469999999999999	71.7683557394002\\
0.480000000000004	72.0785935884178\\
0.489999999999995	72.0785935884178\\
0.5	72.3888314374354\\
0.530000000000001	72.3888314374354\\
0.540000000000006	72.699069286453\\
1	72.699069286453\\
};
\addplot [color=mycolor2, smooth, line width=\subplotlinewidth, forget plot]
  table[row sep=crcr]{%
0	4.84581497797357\\
0.0100000000000051	9.1042584434655\\
0.019999999999996	25.4038179148311\\
0.0300000000000011	42.8046989720999\\
0.0400000000000063	57.0484581497797\\
0.0499999999999972	67.5477239353891\\
0.0600000000000023	75.5506607929515\\
0.0699999999999932	81.791483113069\\
0.0799999999999983	87.5183553597651\\
0.0900000000000034	90.4552129221733\\
0.0999999999999943	92.511013215859\\
0.109999999999999	94.1262848751836\\
0.120000000000005	95.0807635829662\\
0.129999999999995	96.4757709251101\\
0.140000000000001	97.3568281938326\\
0.150000000000006	97.503671071953\\
0.170000000000002	98.0910425844346\\
0.180000000000007	98.5315712187959\\
0.200000000000003	98.5315712187959\\
0.209999999999994	98.6784140969163\\
0.420000000000002	98.6784140969163\\
0.430000000000007	98.8986784140969\\
0.439999999999998	99.0455212922173\\
0.469999999999999	99.0455212922173\\
0.480000000000004	99.265785609398\\
0.489999999999995	99.3392070484581\\
0.5	99.7797356828194\\
0.530000000000001	99.7797356828194\\
0.540000000000006	100\\
1	100\\
};
\end{axis}
&
\begin{axis}[%
width=\subplotwidth,
height=\subplotheight,
scale only axis,
xmin=0,
xmax=1,
xtick={  0, 0.5,   1},
ymin=0,
ymax=103,
ytick={  0,  50, 100},
title={gorilla},
axis x line*=bottom,
axis y line*=left
]
\addplot [color=mycolor1, smooth, line width=\subplotlinewidth, forget plot]
  table[row sep=crcr]{%
0	4.46998722860792\\
0.0100000000000051	7.79054916985952\\
0.019999999999996	18.6462324393359\\
0.0300000000000011	26.5644955300128\\
0.0400000000000063	37.5478927203065\\
0.0499999999999972	46.4878671775224\\
0.0600000000000023	52.6181353767561\\
0.0699999999999932	60.1532567049808\\
0.0799999999999983	64.7509578544061\\
0.0900000000000034	67.0498084291188\\
0.0999999999999943	71.0089399744572\\
0.109999999999999	73.690932311622\\
0.120000000000005	76.5006385696041\\
0.129999999999995	77.5223499361431\\
0.140000000000001	78.9272030651341\\
0.150000000000006	80.2043422733078\\
0.159999999999997	81.8646232439336\\
0.170000000000002	83.2694763729246\\
0.180000000000007	83.7803320561941\\
0.189999999999998	84.6743295019157\\
0.200000000000003	84.8020434227331\\
0.209999999999994	84.8020434227331\\
0.219999999999999	84.9297573435504\\
0.230000000000004	84.9297573435504\\
0.239999999999995	85.3128991060025\\
0.25	85.4406130268199\\
0.260000000000005	86.3346104725415\\
0.290000000000006	86.3346104725415\\
0.299999999999997	87.1008939974457\\
0.359999999999999	87.1008939974457\\
0.370000000000005	87.3563218390805\\
0.459999999999994	87.3563218390805\\
0.469999999999999	87.6117496807152\\
0.489999999999995	87.6117496807152\\
0.5	87.8671775223499\\
0.510000000000005	87.8671775223499\\
0.519999999999996	87.9948914431673\\
0.530000000000001	87.9948914431673\\
0.540000000000006	88.1226053639847\\
0.569999999999993	88.1226053639847\\
0.579999999999998	88.6334610472542\\
0.629999999999995	88.6334610472542\\
0.640000000000001	88.7611749680715\\
1	88.7611749680715\\
};
\addplot [color=mycolor2, smooth, line width=\subplotlinewidth, forget plot]
  table[row sep=crcr]{%
0	4.63362068965517\\
0.0100000000000051	7.97413793103448\\
0.019999999999996	18.1034482758621\\
0.0300000000000011	27.6939655172414\\
0.0400000000000063	41.3793103448276\\
0.0499999999999972	50.8620689655172\\
0.0600000000000023	59.051724137931\\
0.0699999999999932	67.0258620689655\\
0.0799999999999983	73.9224137931034\\
0.0900000000000034	78.125\\
0.0999999999999943	83.1896551724138\\
0.109999999999999	86.0991379310345\\
0.120000000000005	88.4698275862069\\
0.140000000000001	91.9181034482759\\
0.150000000000006	92.9956896551724\\
0.159999999999997	94.5043103448276\\
0.170000000000002	95.6896551724138\\
0.180000000000007	96.1206896551724\\
0.189999999999998	96.6594827586207\\
0.200000000000003	96.7672413793103\\
0.209999999999994	96.7672413793103\\
0.219999999999999	96.875\\
0.230000000000004	96.875\\
0.25	97.5215517241379\\
0.260000000000005	98.2758620689655\\
0.290000000000006	98.2758620689655\\
0.299999999999997	98.9224137931034\\
0.359999999999999	98.9224137931034\\
0.370000000000005	99.1379310344828\\
0.459999999999994	99.1379310344828\\
0.469999999999999	99.3534482758621\\
0.489999999999995	99.3534482758621\\
0.5	99.5689655172414\\
0.510000000000005	99.5689655172414\\
0.519999999999996	99.676724137931\\
0.530000000000001	99.676724137931\\
0.540000000000006	99.7844827586207\\
0.599999999999994	99.7844827586207\\
0.609999999999999	99.8922413793103\\
0.629999999999995	99.8922413793103\\
0.640000000000001	100\\
1	100\\
};
\end{axis}
&
\begin{axis}[%
width=\subplotwidth,
height=\subplotheight,
scale only axis,
xmin=0,
xmax=1,
xtick={  0, 0.5,   1},
ymin=0,
ymax=103,
ytick={  0,  50, 100},
title={horse},
axis x line*=bottom,
axis y line*=left
]
\addplot [color=mycolor1, smooth, line width=\subplotlinewidth, forget plot]
  table[row sep=crcr]{%
0	2.52161383285302\\
0.0100000000000051	4.97118155619597\\
0.019999999999996	11.8155619596542\\
0.0300000000000011	17.507204610951\\
0.0400000000000063	25\\
0.0499999999999972	31.700288184438\\
0.0600000000000023	37.6801152737752\\
0.0699999999999932	42.7953890489914\\
0.0799999999999983	47.4063400576369\\
0.0900000000000034	51.8011527377522\\
0.0999999999999943	55.1152737752162\\
0.109999999999999	56.9164265129683\\
0.120000000000005	59.4380403458213\\
0.129999999999995	61.0230547550432\\
0.140000000000001	62.9682997118156\\
0.150000000000006	63.9769452449568\\
0.159999999999997	65.850144092219\\
0.170000000000002	66.2824207492795\\
0.180000000000007	67.4351585014409\\
0.189999999999998	67.7233429394813\\
0.200000000000003	69.3083573487032\\
0.209999999999994	69.5244956772334\\
0.219999999999999	70.028818443804\\
0.230000000000004	70.1729106628242\\
0.239999999999995	70.1729106628242\\
0.260000000000005	71.0374639769452\\
0.269999999999996	71.3976945244957\\
0.280000000000001	72.3342939481268\\
0.290000000000006	72.6224783861672\\
0.299999999999997	72.6224783861672\\
0.310000000000002	72.7665706051873\\
0.319999999999993	73.342939481268\\
0.329999999999998	73.6311239193084\\
0.340000000000003	73.7031700288184\\
0.349999999999994	73.8472622478386\\
0.359999999999999	73.9193083573487\\
0.370000000000005	73.9193083573487\\
0.390000000000001	74.0634005763689\\
0.400000000000006	74.5677233429395\\
0.420000000000002	74.5677233429395\\
0.430000000000007	74.6397694524496\\
0.590000000000003	74.6397694524496\\
0.599999999999994	74.7118155619597\\
1	74.7118155619597\\
};
\addplot [color=mycolor2, smooth, line width=\subplotlinewidth, forget plot]
  table[row sep=crcr]{%
0	4.53842186694172\\
0.0100000000000051	7.7359463641052\\
0.019999999999996	14.2341413099536\\
0.0300000000000011	22.4858174316658\\
0.0400000000000063	31.0985043837029\\
0.0600000000000023	44.7137699845281\\
0.0699999999999932	48.8911810211449\\
0.0799999999999983	54.1000515729758\\
0.0900000000000034	57.7617328519856\\
0.0999999999999943	60.7529654461062\\
0.109999999999999	63.1253223310985\\
0.120000000000005	65.9618359979371\\
0.129999999999995	68.6952037132542\\
0.140000000000001	70.5518308406395\\
0.150000000000006	71.6864363073749\\
0.159999999999997	73.1304796286746\\
0.170000000000002	74.1619391438886\\
0.180000000000007	74.9355337802991\\
0.189999999999998	75.4512635379061\\
0.200000000000003	76.1732851985559\\
0.209999999999994	77.20474471377\\
0.219999999999999	77.7720474471377\\
0.230000000000004	78.0299123259412\\
0.239999999999995	78.3909231562661\\
0.25	78.9066529138731\\
0.260000000000005	79.3708096957194\\
0.269999999999996	79.9896854048478\\
0.280000000000001	80.2991232594121\\
0.290000000000006	80.9695719443012\\
0.299999999999997	81.3821557503868\\
0.310000000000002	81.4337287261475\\
0.319999999999993	82.3104693140794\\
0.329999999999998	82.5683341928829\\
0.340000000000003	83.1356369262506\\
0.349999999999994	83.8060856111398\\
0.359999999999999	84.2186694172254\\
0.370000000000005	84.5281072717896\\
0.379999999999995	85.146982980918\\
0.390000000000001	85.5595667870036\\
0.400000000000006	86.178442496132\\
0.409999999999997	86.694172253739\\
0.420000000000002	87.1067560598246\\
0.430000000000007	87.7772047447138\\
0.450000000000003	88.2929345023208\\
0.459999999999994	88.602372356885\\
0.469999999999999	88.7570912841671\\
0.480000000000004	89.0149561629706\\
0.489999999999995	89.8916967509025\\
0.5	90.3558535327488\\
0.510000000000005	90.9231562661165\\
0.530000000000001	91.4388860237236\\
0.540000000000006	92.1609076843734\\
0.549999999999997	92.5219185146983\\
0.560000000000002	92.7797833935018\\
0.579999999999998	93.3986591026302\\
0.590000000000003	93.4502320783909\\
0.599999999999994	93.9143888602372\\
0.609999999999999	94.2753996905621\\
0.620000000000005	94.4301186178442\\
0.629999999999995	94.7395564724085\\
0.640000000000001	95.513151108819\\
0.650000000000006	95.6162970603404\\
0.659999999999997	95.8225889633832\\
0.670000000000002	96.1320268179474\\
0.680000000000007	96.1835997937081\\
0.689999999999998	96.3898916967509\\
0.700000000000003	96.4414646725116\\
0.709999999999994	96.9571944301186\\
0.719999999999999	97.3697782362042\\
0.730000000000004	97.7307890665291\\
0.739999999999995	97.9886539453326\\
0.75	98.1433728726147\\
0.760000000000005	98.5043837029397\\
0.769999999999996	98.5559566787004\\
0.780000000000001	98.8138215575039\\
0.790000000000006	98.9169675090253\\
0.799999999999997	99.1232594120681\\
0.810000000000002	99.1748323878288\\
0.819999999999993	99.2779783393502\\
0.829999999999998	99.2779783393502\\
0.840000000000003	99.3295513151109\\
0.849999999999994	99.4326972666323\\
0.859999999999999	99.4326972666323\\
0.870000000000005	99.5874161939144\\
0.879999999999995	99.5874161939144\\
0.890000000000001	99.6905621454358\\
0.900000000000006	100\\
1	100\\
};
\end{axis}
\\
\begin{axis}[%
width=\subplotwidth,
height=\subplotheight,
scale only axis,
xmin=0,
xmax=1,
xtick={  0, 0.5,   1},
ymin=0,
ymax=103,
ytick={  0,  50, 100},
title={michael},
axis x line*=bottom,
axis y line*=left
]
\addplot [color=mycolor1, smooth, line width=\subplotlinewidth, forget plot]
  table[row sep=crcr]{%
0	3.17820658342792\\
0.0100000000000051	6.66855845629966\\
0.019999999999996	16.6004540295119\\
0.0300000000000011	26.7593643586833\\
0.0400000000000063	38.7911464245176\\
0.0499999999999972	47.0771850170261\\
0.0600000000000023	54.3700340522134\\
0.0699999999999932	59.6765039727582\\
0.0799999999999983	63.5073779795687\\
0.0900000000000034	66.7139614074915\\
0.0999999999999943	69.3813847900114\\
0.109999999999999	71.6799091940976\\
0.120000000000005	73.410896708286\\
0.129999999999995	74.9432463110102\\
0.140000000000001	75.9931895573212\\
0.150000000000006	76.7309875141884\\
0.159999999999997	76.9296254256527\\
0.170000000000002	77.3836549375709\\
0.180000000000007	77.5539160045403\\
0.189999999999998	78.0079455164586\\
0.200000000000003	78.0930760499432\\
0.209999999999994	78.2349602724177\\
0.230000000000004	78.2917139614075\\
0.239999999999995	78.3484676503973\\
0.290000000000006	78.3484676503973\\
0.299999999999997	78.4619750283769\\
0.319999999999993	78.4619750283769\\
0.329999999999998	78.5471055618615\\
0.340000000000003	78.5471055618615\\
0.349999999999994	78.6322360953462\\
0.390000000000001	78.6322360953462\\
0.400000000000006	78.7173666288309\\
0.430000000000007	78.7173666288309\\
0.439999999999998	78.7457434733258\\
0.450000000000003	78.7457434733258\\
0.459999999999994	78.8308740068104\\
0.480000000000004	78.8308740068104\\
0.489999999999995	78.8592508513053\\
0.5	78.94438138479\\
0.510000000000005	78.94438138479\\
0.519999999999996	79.0011350737798\\
0.599999999999994	79.0011350737798\\
0.609999999999999	79.1146424517594\\
0.620000000000005	79.1713961407492\\
0.629999999999995	79.1713961407492\\
0.640000000000001	79.199772985244\\
0.670000000000002	79.3700340522134\\
0.680000000000007	79.5119182746879\\
0.689999999999998	79.5119182746879\\
0.700000000000003	79.5970488081725\\
0.709999999999994	79.6538024971623\\
0.719999999999999	79.6538024971623\\
0.739999999999995	79.8240635641317\\
0.829999999999998	79.8240635641317\\
0.840000000000003	79.9091940976163\\
0.859999999999999	79.9091940976163\\
0.870000000000005	80.0227014755959\\
1	80.0227014755959\\
};
\addplot [color=mycolor2, smooth, line width=\subplotlinewidth, forget plot]
  table[row sep=crcr]{%
0	3.57450473729544\\
0.0100000000000051	7.17054263565892\\
0.019999999999996	19.6813092161929\\
0.0400000000000063	45.9302325581395\\
0.0499999999999972	56.9121447028424\\
0.0600000000000023	65.2239448751077\\
0.0699999999999932	71.1886304909561\\
0.0799999999999983	75.8182601205857\\
0.0900000000000034	79.823428079242\\
0.0999999999999943	82.8596037898363\\
0.109999999999999	85.3574504737295\\
0.120000000000005	87.5538329026701\\
0.129999999999995	89.3195521102498\\
0.140000000000001	90.6761412575366\\
0.150000000000006	91.4082687338501\\
0.159999999999997	92.0327304048234\\
0.170000000000002	92.5925925925926\\
0.180000000000007	93.0663221360896\\
0.189999999999998	93.5185185185185\\
0.200000000000003	93.7769164513351\\
0.209999999999994	94.0783807062877\\
0.219999999999999	94.5736434108527\\
0.230000000000004	94.8320413436692\\
0.239999999999995	94.8966408268734\\
0.25	95.1335055986219\\
0.269999999999996	95.219638242894\\
0.280000000000001	95.4134366925065\\
0.310000000000002	95.671834625323\\
0.319999999999993	95.693367786391\\
0.329999999999998	95.8225667527993\\
0.349999999999994	95.9948320413437\\
0.359999999999999	96.0378983634798\\
0.370000000000005	96.1886304909561\\
0.379999999999995	96.1886304909561\\
0.390000000000001	96.2532299741602\\
0.400000000000006	96.4685615848407\\
0.409999999999997	96.5546942291128\\
0.420000000000002	96.6623600344531\\
0.439999999999998	96.7915590008613\\
0.450000000000003	96.8992248062015\\
0.469999999999999	97.0284237726098\\
0.480000000000004	97.0284237726098\\
0.489999999999995	97.13608957795\\
0.5	97.2222222222222\\
0.510000000000005	97.2222222222222\\
0.519999999999996	97.2868217054263\\
0.530000000000001	97.2868217054263\\
0.540000000000006	97.4160206718346\\
0.549999999999997	97.4160206718346\\
0.560000000000002	97.4806201550387\\
0.569999999999993	97.5667527993109\\
0.579999999999998	97.609819121447\\
0.599999999999994	97.609819121447\\
0.620000000000005	97.7820844099914\\
0.629999999999995	97.8251507321275\\
0.640000000000001	98.1050818260121\\
0.650000000000006	98.1912144702842\\
0.659999999999997	98.3419465977606\\
0.670000000000002	98.4496124031008\\
0.680000000000007	98.6864771748493\\
0.689999999999998	98.7295434969854\\
0.700000000000003	98.9448751076658\\
0.709999999999994	99.1386735572782\\
0.719999999999999	99.1386735572782\\
0.75	99.3324720068906\\
0.769999999999996	99.3324720068906\\
0.780000000000001	99.3540051679587\\
0.790000000000006	99.3540051679587\\
0.799999999999997	99.4401378122308\\
0.810000000000002	99.5478036175711\\
0.819999999999993	99.5478036175711\\
0.829999999999998	99.6554694229113\\
0.840000000000003	99.8277347114556\\
0.859999999999999	99.8708010335917\\
0.870000000000005	99.9354005167959\\
0.890000000000001	99.9354005167959\\
0.900000000000006	99.978466838932\\
0.950000000000003	99.978466838932\\
0.959999999999994	100\\
1	100\\
};
\end{axis}
&
\begin{axis}[%
width=\subplotwidth,
height=\subplotheight,
scale only axis,
xmin=0,
xmax=1,
xtick={  0, 0.5,   1},
ymin=0,
ymax=103,
ytick={  0,  50, 100},
title={victoria},
axis x line*=bottom,
axis y line*=left
]
\addplot [color=mycolor1, smooth, line width=\subplotlinewidth, forget plot]
  table[row sep=crcr]{%
0	1.8649193548387\\
0.0100000000000051	3.62903225806451\\
0.019999999999996	8.71975806451613\\
0.0300000000000011	16.4314516129032\\
0.0400000000000063	24.6975806451613\\
0.0499999999999972	31.8548387096774\\
0.0600000000000023	36.0383064516129\\
0.0699999999999932	41.5322580645161\\
0.0799999999999983	44.9092741935484\\
0.0900000000000034	47.3790322580645\\
0.0999999999999943	49.1431451612903\\
0.109999999999999	51.0584677419355\\
0.120000000000005	52.2681451612903\\
0.129999999999995	54.3346774193548\\
0.140000000000001	56.3508064516129\\
0.150000000000006	57.1572580645161\\
0.159999999999997	58.366935483871\\
0.170000000000002	59.0221774193548\\
0.180000000000007	59.929435483871\\
0.189999999999998	60.5846774193548\\
0.200000000000003	61.5927419354839\\
0.219999999999999	63.0040322580645\\
0.239999999999995	63.7096774193548\\
0.25	64.2641129032258\\
0.260000000000005	64.7681451612903\\
0.269999999999996	65.3225806451613\\
0.280000000000001	65.8266129032258\\
0.290000000000006	67.0362903225806\\
0.310000000000002	69.1532258064516\\
0.319999999999993	69.5564516129032\\
0.329999999999998	70.3125\\
0.340000000000003	71.0181451612903\\
0.349999999999994	71.6229838709677\\
0.359999999999999	72.2782258064516\\
0.370000000000005	72.7822580645161\\
0.379999999999995	73.3366935483871\\
0.390000000000001	74.3447580645161\\
0.400000000000006	74.6471774193548\\
0.409999999999997	75.5040322580645\\
0.420000000000002	76.1592741935484\\
0.430000000000007	76.4616935483871\\
0.439999999999998	76.7137096774193\\
0.450000000000003	76.9153225806452\\
0.459999999999994	76.9153225806452\\
0.469999999999999	77.4193548387097\\
0.480000000000004	77.8225806451613\\
0.489999999999995	78.3770161290323\\
0.5	78.679435483871\\
0.519999999999996	79.3850806451613\\
0.530000000000001	79.7883064516129\\
0.549999999999997	80.1915322580645\\
0.560000000000002	80.3427419354839\\
0.579999999999998	80.3427419354839\\
0.590000000000003	80.4435483870968\\
0.739999999999995	80.4435483870968\\
0.75	80.5947580645161\\
0.769999999999996	80.5947580645161\\
0.780000000000001	80.695564516129\\
0.790000000000006	80.695564516129\\
0.799999999999997	80.8467741935484\\
0.810000000000002	80.8971774193548\\
0.819999999999993	80.9979838709677\\
0.900000000000006	80.9979838709677\\
0.909999999999997	81.0483870967742\\
0.920000000000002	81.0483870967742\\
0.930000000000007	81.0987903225806\\
1	81.0987903225806\\
};
\addplot [color=mycolor2, smooth, line width=\subplotlinewidth, forget plot]
  table[row sep=crcr]{%
0	2.38774055595154\\
0.0100000000000051	5.45260156806843\\
0.019999999999996	14.2195295794726\\
0.0400000000000063	37.9187455452602\\
0.0499999999999972	49.9643620812545\\
0.0600000000000023	58.0185317177477\\
0.0699999999999932	65.2886671418389\\
0.0799999999999983	71.6678545972915\\
0.0900000000000034	75.5880256593015\\
0.0999999999999943	78.5459729151817\\
0.109999999999999	80.7911617961511\\
0.120000000000005	81.7890235210264\\
0.129999999999995	83.1432644333571\\
0.140000000000001	85.1033499643621\\
0.150000000000006	85.8873841767641\\
0.159999999999997	87.4198146828225\\
0.170000000000002	88.2038488952245\\
0.180000000000007	88.952245188881\\
0.189999999999998	89.2017106200998\\
0.200000000000003	89.736279401283\\
0.209999999999994	90.0213827512473\\
0.219999999999999	90.1639344262295\\
0.230000000000004	90.2708481824661\\
0.239999999999995	90.4846756949394\\
0.25	90.7341411261582\\
0.260000000000005	90.7697790449038\\
0.269999999999996	90.912330719886\\
0.280000000000001	90.9479686386315\\
0.310000000000002	90.9479686386315\\
0.319999999999993	91.3399857448325\\
0.329999999999998	91.6250890947969\\
0.340000000000003	91.8032786885246\\
0.349999999999994	92.1240199572345\\
0.359999999999999	92.6585887384177\\
0.370000000000005	93.335709194583\\
0.379999999999995	93.4426229508197\\
0.390000000000001	94.119743406985\\
0.400000000000006	94.2622950819672\\
0.420000000000002	94.4761225944405\\
0.430000000000007	94.7255880256593\\
0.439999999999998	94.9394155381326\\
0.450000000000003	95.3314326443336\\
0.459999999999994	95.6521739130435\\
0.469999999999999	95.7590876692801\\
0.480000000000004	96.1511047754811\\
0.489999999999995	96.5074839629366\\
0.5	96.5787598004277\\
0.510000000000005	96.7569493941554\\
0.530000000000001	96.9707769066287\\
0.540000000000006	97.1846044191019\\
0.549999999999997	97.255880256593\\
0.569999999999993	97.3271560940841\\
0.579999999999998	97.3271560940841\\
0.590000000000003	97.6122594440485\\
0.620000000000005	97.7191732002851\\
0.640000000000001	97.8617248752673\\
0.650000000000006	97.8617248752673\\
0.659999999999997	97.8973627940128\\
0.709999999999994	97.8973627940128\\
0.719999999999999	98.1111903064861\\
0.730000000000004	98.2537419814683\\
0.739999999999995	98.2537419814683\\
0.75	98.3962936564505\\
0.760000000000005	98.431931575196\\
0.769999999999996	98.6457590876693\\
0.780000000000001	98.7883107626515\\
0.790000000000006	98.8595866001426\\
0.799999999999997	99.0734141126158\\
0.819999999999993	99.3585174625802\\
0.840000000000003	99.5010691375624\\
0.859999999999999	99.5010691375624\\
0.879999999999995	99.6436208125445\\
0.890000000000001	99.9287241625089\\
0.900000000000006	99.9287241625089\\
0.909999999999997	99.9643620812544\\
0.920000000000002	99.9643620812544\\
0.930000000000007	100\\
1	100\\
};
\end{axis}

&
\begin{axis}[%
width=\subplotwidth,
height=\subplotheight,
scale only axis,
xmin=0,
xmax=1,
xtick={  0, 0.5,   1},
ymin=0,
ymax=103,
ytick={  0,  50, 100},
title={wolf},
axis x line*=bottom,
axis y line*=left
]
\addplot [color=mycolor1, smooth, line width=\subplotlinewidth, forget plot]
  table[row sep=crcr]{%
0	3.09278350515464\\
0.00999999999999801	3.35051546391752\\
0.0200000000000031	9.79381443298969\\
0.0300000000000011	15.4639175257732\\
0.0399999999999991	20.8762886597938\\
0.0499999999999972	26.5463917525773\\
0.0600000000000023	37.8865979381443\\
0.0700000000000003	41.4948453608247\\
0.0799999999999983	45.360824742268\\
0.0900000000000034	47.1649484536082\\
0.100000000000001	48.1958762886598\\
0.109999999999999	50\\
0.119999999999997	51.5463917525773\\
0.130000000000003	51.5463917525773\\
0.140000000000001	51.8041237113402\\
0.149999999999999	52.8350515463918\\
0.159999999999997	53.6082474226804\\
0.170000000000002	53.8659793814433\\
0.18	54.639175257732\\
1	54.639175257732\\
};
\addplot [color=mycolor2, line width=\subplotlinewidth, forget plot]
  table[row sep=crcr]{%
0	2.19941348973607\\
0.0100000000000051	4.83870967741936\\
0.019999999999996	15.5425219941349\\
0.0300000000000011	24.3401759530792\\
0.0400000000000063	33.5777126099707\\
0.0499999999999972	44.574780058651\\
0.0600000000000023	56.7448680351906\\
0.0799999999999983	68.1818181818182\\
0.0900000000000034	71.8475073313783\\
0.0999999999999943	74.9266862170088\\
0.109999999999999	77.2727272727273\\
0.120000000000005	79.7653958944282\\
0.129999999999995	81.8181818181818\\
0.140000000000001	83.4310850439883\\
0.150000000000006	86.3636363636364\\
0.170000000000002	89.2961876832845\\
0.180000000000007	90.0293255131965\\
0.189999999999998	90.6158357771261\\
0.200000000000003	91.3489736070381\\
0.209999999999994	92.5219941348974\\
0.230000000000004	94.574780058651\\
0.239999999999995	94.574780058651\\
0.25	95.1612903225807\\
0.260000000000005	95.1612903225807\\
0.280000000000001	95.7478005865103\\
0.299999999999997	96.6275659824047\\
0.310000000000002	97.2140762463343\\
0.329999999999998	97.2140762463343\\
0.340000000000003	98.3870967741936\\
0.390000000000001	98.3870967741936\\
0.400000000000006	98.9736070381232\\
0.409999999999997	99.266862170088\\
0.439999999999998	99.266862170088\\
0.450000000000003	99.4134897360704\\
0.459999999999994	99.7067448680352\\
0.489999999999995	99.7067448680352\\
0.5	100\\
1	100\\
};
\end{axis}
\\
};

\end{tikzpicture}%

%% file: figures/pckKids.tex
\definecolor{mycolor1}{rgb}{0.00000,0.44706,0.74118}%
\definecolor{mycolor2}{rgb}{0.85098,0.32549,0.09804}%

\definecolor{mycolor1}{rgb}{0.00000,0.44706,0.74118}%
\definecolor{mycolor2}{rgb}{0.92941,0.69412,0.12549}%
\definecolor{mycolor3}{rgb}{0.85098,0.32549,0.09804}%
\definecolor{mycolor1}{rgb}{0.00000,0.44706,0.74118}%
\definecolor{mycolor2}{rgb}{0.92941,0.69412,0.12549}%
\definecolor{mycolor3}{rgb}{0.85098,0.32549,0.09804}%
\begin{tikzpicture}
\pgfplotsset{%
    width=1.3\textwidth,
    label style = {font=\tiny},
    legend style = {font=\scriptsize},
    tick label style = {font=\tiny},
    title style =  {font=\scriptsize\bfseries},
    every axis label = {font=\tiny},
}
\matrix{
	\begin{axis}[%
		legend columns=3,
		name=cat,
		width=\subplotwidth,
		height=\subplotheight,
		scale only axis,
		xmin=0,
		xmax=1,
		xtick={  0, 0.5,   1},
		ymin=0,
		ymax=103,
		ytick={  0,  50, 100},
		title={kid},
		axis x line*=bottom,
		axis y line*=left,
		legend style={legend cell align=left, align=left,
		legend to name=leg:sotaVsOurs,
		draw=white!15!black}
		]
		\addplot [color=mycolor1, smooth, line width=\subplotlinewidth]
		table[row sep=crcr]{%
            0	4.98349834983498\\
            0.01	7.98679867986799\\
            0.02	17.6567656765677\\
            0.03	27.0627062706271\\
            0.04	33.8613861386139\\
            0.05	39.8349834983498\\
            0.06	44.3234323432343\\
            0.07	48.0528052805281\\
            0.08	51.023102310231\\
            0.09	53.3993399339934\\
            0.1	55.3795379537954\\
            0.11	57.2277227722772\\
            0.12	59.042904290429\\
            0.13	60.7590759075908\\
            0.14	62.5082508250825\\
            0.15	64.2244224422442\\
            0.16	65.5115511551155\\
            0.17	67.1287128712871\\
            0.18	68.2508250825082\\
            0.19	69.7689768976898\\
            0.2	71.0891089108911\\
            0.21	72.3102310231023\\
            0.22	73.7953795379538\\
            0.23	74.7524752475248\\
            0.24	75.6435643564356\\
            0.25	76.2706270627063\\
            0.26	76.7656765676568\\
            0.27	77.2277227722772\\
            0.28	77.7557755775578\\
            0.29	78.019801980198\\
            0.3	78.4488448844884\\
            0.31	78.6138613861386\\
            0.32	78.8448844884488\\
            0.33	79.2079207920792\\
            0.34	79.5709570957096\\
            0.35	79.8019801980198\\
            0.36	80.1320132013201\\
            0.37	80.5940594059406\\
            0.38	81.0891089108911\\
            0.39	81.6171617161716\\
            0.4	81.8811881188119\\
            0.41	82.3432343234323\\
            0.42	82.7722772277228\\
            0.43	83.3003300330033\\
            0.44	83.8943894389439\\
            0.45	84.2904290429043\\
            0.46	84.7524752475248\\
            0.47	85.1155115511551\\
            0.48	85.6105610561056\\
            0.49	86.2376237623762\\
            0.5	86.7986798679868\\
            0.51	87.2607260726073\\
            0.52	87.7887788778878\\
            0.53	88.3168316831683\\
            0.54	88.7458745874588\\
            0.55	89.3069306930693\\
            0.56	89.7689768976898\\
            0.57	89.96699669967\\
            0.58	90.0990099009901\\
            0.59	90.3630363036304\\
            0.6	90.9240924092409\\
            0.61	91.3201320132013\\
            0.62	91.6171617161716\\
            0.63	91.947194719472\\
            0.64	92.5082508250825\\
            0.65	92.8712871287129\\
            0.66	93.3993399339934\\
            0.67	93.7953795379538\\
            0.68	94.059405940594\\
            0.69	94.4884488448845\\
            0.7	95.016501650165\\
            0.71	95.2475247524752\\
            0.72	95.6435643564356\\
            0.73	95.973597359736\\
            0.74	96.4356435643564\\
            0.75	96.6666666666667\\
            0.76	97.029702970297\\
            0.77	97.3267326732673\\
            0.78	97.7227722772277\\
            0.79	98.1518151815182\\
            0.8	98.3828382838284\\
            0.81	98.7788778877888\\
            0.82	99.009900990099\\
            0.83	99.3069306930693\\
            0.84	99.4389438943894\\
            0.85	99.4389438943894\\
            0.86	99.5709570957096\\
            0.87	99.6369636963696\\
            0.88	99.7359735973597\\
            0.89	99.7689768976898\\
            0.9	99.7689768976898\\
            0.91	99.7689768976898\\
            0.92	99.8019801980198\\
            0.93	99.9339933993399\\
            0.94	99.96699669967\\
            0.95	99.96699669967\\
            0.96	99.96699669967\\
            0.97	99.96699669967\\
            0.98	100\\
            1	100\\
		};
		\addlegendentry{{\textcolor{black}{Eisenberger et al. '20}}}
		
		\addplot [color=mycolor2, smooth, line width=\subplotlinewidth]
		table[row sep=crcr]{%
            0	4.25742574257426\\
            0.01	6.6996699669967\\
            0.02	14.8514851485149\\
            0.03	22.9372937293729\\
            0.04	30.8910891089109\\
            0.05	37.9207920792079\\
            0.06	43.6303630363036\\
            0.07	48.4818481848185\\
            0.08	52.3432343234323\\
            0.09	55.2475247524753\\
            0.1	57.6567656765677\\
            0.11	59.1419141914191\\
            0.12	60.3300330033003\\
            0.13	61.1551155115512\\
            0.14	62.3102310231023\\
            0.15	62.9042904290429\\
            0.16	63.5973597359736\\
            0.17	64.3234323432343\\
            0.18	65.2475247524752\\
            0.19	66.2046204620462\\
            0.2	66.8646864686469\\
            0.21	68.2178217821782\\
            0.22	69.5709570957096\\
            0.23	70.4620462046205\\
            0.24	71.4851485148515\\
            0.25	72.2442244224422\\
            0.26	72.8712871287129\\
            0.27	73.6633663366337\\
            0.28	74.1914191419142\\
            0.29	74.6534653465347\\
            0.3	75.1485148514851\\
            0.31	75.5445544554455\\
            0.32	75.973597359736\\
            0.33	76.4026402640264\\
            0.34	76.6666666666667\\
            0.35	77.1617161716172\\
            0.36	77.6237623762376\\
            0.37	78.019801980198\\
            0.38	78.4818481848185\\
            0.39	78.8448844884488\\
            0.4	79.4719471947195\\
            0.41	80.1320132013201\\
            0.42	80.7260726072607\\
            0.43	81.3201320132013\\
            0.44	82.046204620462\\
            0.45	82.6402640264026\\
            0.46	83.3333333333333\\
            0.47	83.7293729372937\\
            0.48	84.5544554455445\\
            0.49	85.049504950495\\
            0.5	85.4455445544554\\
            0.51	85.9075907590759\\
            0.52	86.3036303630363\\
            0.53	86.8316831683168\\
            0.54	87.1617161716172\\
            0.55	87.3597359735974\\
            0.56	87.6237623762376\\
            0.57	88.0858085808581\\
            0.58	88.4818481848185\\
            0.59	89.009900990099\\
            0.6	89.3729372937294\\
            0.61	89.6039603960396\\
            0.62	90\\
            0.63	90.3630363036304\\
            0.64	90.6600660066007\\
            0.65	91.2541254125413\\
            0.66	91.9141914191419\\
            0.67	92.2772277227723\\
            0.68	92.4752475247525\\
            0.69	92.7392739273927\\
            0.7	92.970297029703\\
            0.71	93.2343234323432\\
            0.72	93.5313531353135\\
            0.73	93.8283828382838\\
            0.74	94.2244224422442\\
            0.75	94.5214521452145\\
            0.76	94.7524752475247\\
            0.77	95.049504950495\\
            0.78	95.5445544554455\\
            0.79	96.3036303630363\\
            0.8	96.7986798679868\\
            0.81	97.1617161716172\\
            0.82	97.1947194719472\\
            0.83	97.3267326732673\\
            0.84	97.5247524752475\\
            0.85	97.7557755775578\\
            0.86	97.9207920792079\\
            0.87	98.052805280528\\
            0.88	98.0858085808581\\
            0.89	98.2178217821782\\
            0.9	98.5808580858086\\
            0.91	98.9108910891089\\
            0.92	99.1419141914191\\
            0.93	99.3729372937294\\
            0.94	99.6039603960396\\
            0.95	99.7359735973597\\
            0.96	99.7359735973597\\
            0.97	99.9339933993399\\
            0.98	100\\
            1	100\\
		};
		\addlegendentry{{\textcolor{black}{Ren et al. '21}}}
		
		\addplot [color=mycolor3, smooth, line width=\subplotlinewidth]
		table[row sep=crcr]{%
            0	7.7178729689808\\
            0.0100000000000051	12.5553914327917\\
            0.0300000000000011	43.4145741014279\\
            0.0400000000000063	55.8714918759232\\
            0.0499999999999972	66.260462826194\\
            0.0600000000000023	74.1875923190547\\
            0.0699999999999932	79.3820777941901\\
            0.0799999999999983	83.7764647956672\\
            0.0900000000000034	87.2722796651896\\
            0.0999999999999943	89.5125553914328\\
            0.109999999999999	90.9158050221566\\
            0.120000000000005	91.999015263417\\
            0.129999999999995	92.8729689807976\\
            0.140000000000001	93.4022648941408\\
            0.150000000000006	93.8577055637617\\
            0.159999999999997	94.0177252584934\\
            0.170000000000002	94.2516001969473\\
            0.180000000000007	94.4485475135401\\
            0.189999999999998	94.8301329394387\\
            0.200000000000003	94.9162973904481\\
            0.209999999999994	95.0270802560315\\
            0.219999999999999	95.1994091580502\\
            0.230000000000004	95.3101920236337\\
            0.260000000000005	95.4209748892171\\
            0.269999999999996	95.4948301329394\\
            0.280000000000001	95.4948301329394\\
            0.290000000000006	95.5194485475135\\
            0.299999999999997	95.5317577548006\\
            0.310000000000002	95.630231413097\\
            0.319999999999993	95.6794682422452\\
            0.329999999999998	95.8025603151157\\
            0.340000000000003	95.8394879369769\\
            0.349999999999994	95.8517971442639\\
            0.359999999999999	95.9748892171344\\
            0.370000000000005	95.9995076317085\\
            0.379999999999995	96.0856720827179\\
            0.390000000000001	96.2087641555884\\
            0.400000000000006	96.2580009847366\\
            0.409999999999997	96.4672575086164\\
            0.420000000000002	96.4918759231905\\
            0.450000000000003	96.7503692762186\\
            0.459999999999994	96.8242245199409\\
            0.469999999999999	96.9596258000985\\
            0.480000000000004	97.1688823239783\\
            0.489999999999995	97.3042836041359\\
            0.5	97.3535204332841\\
            0.510000000000005	97.4519940915805\\
            0.519999999999996	97.5627769571639\\
            0.530000000000001	97.575086164451\\
            0.540000000000006	97.575086164451\\
            0.549999999999997	97.5997045790251\\
            0.560000000000002	97.7227966518956\\
            0.569999999999993	97.9689807976366\\
            0.590000000000003	98.1413096996553\\
            0.599999999999994	98.301329394387\\
            0.609999999999999	98.3628754308222\\
            0.620000000000005	98.4121122599705\\
            0.629999999999995	98.5105859182669\\
            0.640000000000001	98.7444608567208\\
            0.650000000000006	98.8429345150172\\
            0.659999999999997	98.9044805514525\\
            0.670000000000002	98.9044805514525\\
            0.680000000000007	98.9167897587395\\
            0.689999999999998	98.9167897587395\\
            0.700000000000003	98.9783357951748\\
            0.709999999999994	99.1014278680453\\
            0.719999999999999	99.1999015263417\\
            0.730000000000004	99.273756770064\\
            0.739999999999995	99.4460856720827\\
            0.75	99.6307237813885\\
            0.760000000000005	99.7291974396849\\
            0.769999999999996	99.8030526834072\\
            0.780000000000001	99.9507631708518\\
            0.790000000000006	99.987690792713\\
            0.799999999999997	100\\
            1	100\\
		};
		\addlegendentry{{\textcolor{black}{Ours}}}
		
	\end{axis}
	&
	\begin{axis}[%
		name=centaur,
		width=\subplotwidth,
		height=\subplotheight,
		scale only axis,
		xmin=0,
		xmax=1,
		xtick={  0, 0.5,   1},
		ymin=0,
		ymax=103,
		ytick={  0,  50, 100},
		title={fat kid},
		axis x line*=bottom,
		axis y line*=left
		]
		\addplot [color=mycolor1, smooth, line width=\subplotlinewidth, forget plot]
		table[row sep=crcr]{%
            0	5.94059405940594\\
            0.0100000000000051	10.03300330033\\
            0.019999999999996	21.7821782178218\\
            0.0300000000000011	33.993399339934\\
            0.0400000000000063	44.8514851485149\\
            0.0499999999999972	53.9273927392739\\
            0.0600000000000023	61.7161716171617\\
            0.0699999999999932	66.6666666666667\\
            0.0799999999999983	71.1881188118812\\
            0.0900000000000034	74.4554455445545\\
            0.0999999999999943	76.7986798679868\\
            0.109999999999999	78.5478547854785\\
            0.120000000000005	80.4620462046205\\
            0.129999999999995	81.6831683168317\\
            0.140000000000001	83.036303630363\\
            0.150000000000006	83.993399339934\\
            0.159999999999997	85.0825082508251\\
            0.170000000000002	86.039603960396\\
            0.180000000000007	86.5346534653465\\
            0.189999999999998	87.2607260726073\\
            0.200000000000003	87.4257425742574\\
            0.209999999999994	87.6567656765677\\
            0.219999999999999	87.986798679868\\
            0.230000000000004	88.6468646864686\\
            0.239999999999995	88.976897689769\\
            0.25	89.4059405940594\\
            0.260000000000005	89.7029702970297\\
            0.269999999999996	90.03300330033\\
            0.290000000000006	90.4290429042904\\
            0.299999999999997	90.5940594059406\\
            0.310000000000002	90.990099009901\\
            0.319999999999993	91.2871287128713\\
            0.329999999999998	91.4191419141914\\
            0.340000000000003	91.6501650165016\\
            0.349999999999994	91.7821782178218\\
            0.359999999999999	92.3102310231023\\
            0.370000000000005	92.4092409240924\\
            0.379999999999995	92.7722772277228\\
            0.390000000000001	92.970297029703\\
            0.400000000000006	93.2673267326733\\
            0.409999999999997	93.4983498349835\\
            0.420000000000002	93.7623762376238\\
            0.439999999999998	94.4884488448845\\
            0.450000000000003	94.7194719471947\\
            0.459999999999994	95.049504950495\\
            0.469999999999999	95.2805280528053\\
            0.480000000000004	95.4455445544554\\
            0.489999999999995	95.6765676567657\\
            0.5	95.7755775577558\\
            0.510000000000005	95.9405940594059\\
            0.519999999999996	96.3366336633663\\
            0.530000000000001	96.6996699669967\\
            0.540000000000006	96.8646864686469\\
            0.549999999999997	97.0627062706271\\
            0.569999999999993	97.3927392739274\\
            0.579999999999998	97.5247524752475\\
            0.590000000000003	97.5907590759076\\
            0.599999999999994	97.7557755775578\\
            0.609999999999999	97.8877887788779\\
            0.620000000000005	97.953795379538\\
            0.629999999999995	98.052805280528\\
            0.650000000000006	98.052805280528\\
            0.659999999999997	98.0858085808581\\
            0.670000000000002	98.2178217821782\\
            0.680000000000007	98.4158415841584\\
            0.689999999999998	98.5148514851485\\
            0.700000000000003	98.5478547854785\\
            0.719999999999999	98.7458745874587\\
            0.730000000000004	98.7788778877888\\
            0.739999999999995	98.8778877887789\\
            0.75	99.0759075907591\\
            0.760000000000005	99.1419141914191\\
            0.769999999999996	99.3399339933993\\
            0.780000000000001	99.4719471947195\\
            0.799999999999997	99.6699669966997\\
            0.810000000000002	99.7029702970297\\
            0.849999999999994	99.7029702970297\\
            0.859999999999999	99.8349834983498\\
            0.879999999999995	99.96699669967\\
            0.890000000000001	99.96699669967\\
            0.900000000000006	100\\
		};
		\addplot [color=mycolor2, smooth, line width=\subplotlinewidth, forget plot]
		table[row sep=crcr]{%
            0	1.51815181518151\\
            0.0100000000000051	2.17821782178218\\
            0.019999999999996	5.01650165016501\\
            0.0300000000000011	8.44884488448845\\
            0.0400000000000063	11.980198019802\\
            0.0499999999999972	15.3135313531353\\
            0.0600000000000023	18.4818481848185\\
            0.0799999999999983	23.8283828382838\\
            0.0900000000000034	26.1056105610561\\
            0.0999999999999943	28.2178217821782\\
            0.109999999999999	30.2640264026403\\
            0.120000000000005	32.1122112211221\\
            0.129999999999995	33.6303630363036\\
            0.140000000000001	35.4455445544555\\
            0.150000000000006	37.0957095709571\\
            0.159999999999997	38.4158415841584\\
            0.170000000000002	39.7689768976898\\
            0.180000000000007	41.1551155115512\\
            0.189999999999998	42.7722772277228\\
            0.200000000000003	44.3564356435644\\
            0.209999999999994	45.6765676567657\\
            0.219999999999999	47.1947194719472\\
            0.230000000000004	49.042904290429\\
            0.239999999999995	50.2310231023102\\
            0.25	51.5511551155115\\
            0.269999999999996	53.993399339934\\
            0.290000000000006	57.2277227722772\\
            0.299999999999997	58.6468646864686\\
            0.310000000000002	60\\
            0.319999999999993	61.5181518151815\\
            0.329999999999998	63.1023102310231\\
            0.340000000000003	64.5874587458746\\
            0.349999999999994	66.1716171617162\\
            0.359999999999999	67.6567656765677\\
            0.370000000000005	68.6138613861386\\
            0.379999999999995	69.6369636963696\\
            0.390000000000001	70.7260726072607\\
            0.400000000000006	71.6501650165017\\
            0.409999999999997	73.1353135313531\\
            0.420000000000002	74.1254125412541\\
            0.430000000000007	75.2145214521452\\
            0.439999999999998	76.006600660066\\
            0.450000000000003	76.7326732673267\\
            0.459999999999994	77.4917491749175\\
            0.469999999999999	78.2178217821782\\
            0.480000000000004	78.7458745874587\\
            0.489999999999995	79.1749174917492\\
            0.5	79.6369636963696\\
            0.510000000000005	80.3300330033003\\
            0.519999999999996	81.0561056105611\\
            0.530000000000001	81.5511551155115\\
            0.540000000000006	81.947194719472\\
            0.549999999999997	82.3762376237624\\
            0.560000000000002	82.9372937293729\\
            0.569999999999993	83.7293729372937\\
            0.579999999999998	84.1914191419142\\
            0.590000000000003	84.7854785478548\\
            0.599999999999994	85.5115511551155\\
            0.609999999999999	86.4356435643564\\
            0.620000000000005	86.996699669967\\
            0.629999999999995	87.5247524752475\\
            0.640000000000001	88.019801980198\\
            0.650000000000006	88.8778877887789\\
            0.659999999999997	89.3399339933993\\
            0.670000000000002	90\\
            0.680000000000007	90.5280528052805\\
            0.709999999999994	92.3102310231023\\
            0.719999999999999	92.8382838283828\\
            0.739999999999995	93.7623762376238\\
            0.75	94.3894389438944\\
            0.760000000000005	94.8514851485149\\
            0.769999999999996	95.2805280528053\\
            0.780000000000001	95.6435643564356\\
            0.790000000000006	95.973597359736\\
            0.799999999999997	96.5346534653465\\
            0.810000000000002	97.0627062706271\\
            0.819999999999993	97.4257425742574\\
            0.829999999999998	97.8547854785479\\
            0.840000000000003	98.1848184818482\\
            0.849999999999994	98.3168316831683\\
            0.870000000000005	98.6468646864686\\
            0.879999999999995	98.8778877887789\\
            0.890000000000001	99.009900990099\\
            0.900000000000006	99.0759075907591\\
            0.909999999999997	99.2079207920792\\
            0.920000000000002	99.4059405940594\\
            0.930000000000007	99.4059405940594\\
            0.939999999999998	99.5049504950495\\
            0.950000000000003	99.7029702970297\\
            0.959999999999994	99.8349834983498\\
            0.969999999999999	99.9009900990099\\
            1	100\\
		};
		\addplot [color=mycolor3, smooth, line width=\subplotlinewidth, forget plot]
		table[row sep=crcr]{%
            0	5.57695050634717\\
            0.0100000000000051	8.9430894308943\\
            0.019999999999996	20.0542005420054\\
            0.0300000000000011	32.3776921979746\\
            0.0400000000000063	43.3176437027528\\
            0.0499999999999972	52.1181001283697\\
            0.0600000000000023	60.0770218228498\\
            0.0699999999999932	67.2514619883041\\
            0.0799999999999983	72.6001996862074\\
            0.0900000000000034	77.0503494508629\\
            0.0999999999999943	80.1312223648552\\
            0.109999999999999	82.7699329624875\\
            0.120000000000005	85.3088004564256\\
            0.129999999999995	87.0061332192269\\
            0.140000000000001	87.776351447725\\
            0.150000000000006	88.3896733704179\\
            0.159999999999997	88.9316787904721\\
            0.170000000000002	89.3025246041934\\
            0.180000000000007	89.7304236200257\\
            0.189999999999998	90.0870061332192\\
            0.200000000000003	90.3009556411354\\
            0.209999999999994	90.543431750107\\
            0.219999999999999	90.7003280559121\\
            0.230000000000004	90.9142775638283\\
            0.239999999999995	91.2565967764941\\
            0.25	91.4990728854657\\
            0.269999999999996	92.0696049065754\\
            0.280000000000001	92.1837113107973\\
            0.290000000000006	92.4547140208244\\
            0.299999999999997	92.9539295392954\\
            0.310000000000002	93.1678790472115\\
            0.319999999999993	93.638567964627\\
            0.329999999999998	94.3089430894309\\
            0.340000000000003	94.6084724005135\\
            0.349999999999994	94.8509485094851\\
            0.359999999999999	95.3073741263728\\
            0.370000000000005	95.464270432178\\
            0.379999999999995	95.7923263443161\\
            0.390000000000001	95.9777492511767\\
            0.400000000000006	96.0775923548709\\
            0.409999999999997	96.2202253601483\\
            0.420000000000002	96.3771216659535\\
            0.430000000000007	96.5910711738696\\
            0.439999999999998	96.7622307802025\\
            0.450000000000003	96.9904435886464\\
            0.459999999999994	97.3898160034232\\
            0.469999999999999	97.4753958065896\\
            0.480000000000004	97.5181857081729\\
            0.489999999999995	97.6608187134503\\
            0.5	97.9175581229497\\
            0.510000000000005	97.9460847240051\\
            0.519999999999996	98.1600342319213\\
            0.530000000000001	98.1885608329767\\
            0.540000000000006	98.1885608329767\\
            0.560000000000002	98.3026672371987\\
            0.569999999999993	98.3739837398374\\
            0.579999999999998	98.4167736414206\\
            0.590000000000003	98.5308800456426\\
            0.599999999999994	98.6021965482813\\
            0.609999999999999	98.8446726572529\\
            0.620000000000005	99.0443588646413\\
            0.629999999999995	99.1299386678077\\
            0.640000000000001	99.1442019683355\\
            0.650000000000006	99.1869918699187\\
            0.659999999999997	99.2155184709742\\
            0.670000000000002	99.500784481529\\
            0.680000000000007	99.5293110825845\\
            0.689999999999998	99.6291541862787\\
            0.700000000000003	99.6576807873342\\
            0.709999999999994	99.7004706889174\\
            0.719999999999999	99.7004706889174\\
            0.730000000000004	99.7289972899729\\
            0.75	99.8145770931394\\
            0.760000000000005	99.8145770931394\\
            0.769999999999996	99.8288403936671\\
            0.780000000000001	99.8288403936671\\
            0.790000000000006	99.8716302952503\\
            0.799999999999997	99.942946797889\\
            0.810000000000002	100\\
            1	100\\
		};
	\end{axis}\\
};
\end{tikzpicture}%